%% file: arxiVersion.tex
\documentclass[10pt,twocolumn,letterpaper]{article}

\usepackage{cvpr}
\usepackage{rotating}
\usepackage{times}
\usepackage{epsfig}
\usepackage{graphicx}
\usepackage{amsmath}
\usepackage{amssymb}
\usepackage{amsfonts}
\usepackage{empheq}
\usepackage{framed, color}
\usepackage{xcolor}
\usepackage{graphics}
\usepackage{graphicx}
\usepackage[]{graphicx} 
\usepackage{epsfig} 
\usepackage{subfigure}
\usepackage{algorithm}
\usepackage{algorithmic}
\usepackage{stmaryrd}
\usepackage[mathscr]{eucal}
\usepackage{lineno}
\usepackage{color}
\usepackage{filecontents}
\usepackage{subfigure}
\usepackage{multirow}
\usepackage{filecontents}
\usepackage{verbatim} 
\usepackage{tabularx}
\usepackage{lineno}
\usepackage{setspace}
\usepackage{multirow}
\usepackage{textcomp,booktabs}
\usepackage{caption}

\def\D{{\bf D}}

\def\K{{\bf K}}

\def\I{{\bf I}}

\def\X{{\bf X}}
\def\Y{{\bf Y}}

\def\q{{\bf q}}

\def\S{{\bf S}}
\def\x{{\bf x}}

\def\M{{\bf M}}

\def\0{{\bf 0}}
\def\1{{\bf 1}}

\def\RB{{\mathbb R}}

\def\diag{\mathsf{diag}}

\newtheorem{propositions}{Proposition}

\newtheorem{proof}{Proof}

\graphicspath{{./figures/}}   

\newcommand*{\colorboxed}{}
\def\colorboxed#1#{%
  \colorboxedAux{#1}%
}
\newcommand*{\colorboxedAux}[3]{%
  \begingroup
    \colorlet{cb@saved}{.}%
    \color#1{#2}%
    \boxed{%
      \color{cb@saved}%
      #3%
    }%
  \endgroup
}

\usepackage[pagebackref=true,breaklinks=true,letterpaper=true,colorlinks,bookmarks=false]{hyperref}
\cvprfinalcopy 
\def\cvprPaperID{533} 

\ifcvprfinal\fi

\begin{document}

\title{Recurrent Pixel Embedding for Instance Grouping}

\author{
Shu Kong, \ \   Charless Fowlkes\\
  Department of Computer Science\\
  University of California, Irvine\\
  Irvine, CA 92697, USA \\
  \texttt{\{skong2, fowlkes\}@ics.uci.edu } \\ \\
  \ [\href{http://www.ics.uci.edu/~skong2/SMMMSG.html}{Project Page}],
[\href{https://github.com/aimerykong/Recurrent-Pixel-Embedding-for-Instance-Grouping}{
Github}],
[\href{http://www.ics.uci.edu/~skong2/slides/pixel_embedding_for_grouping_public_version.pdf}{Slides}],
[\href{http://www.ics.uci.edu/~skong2/slides/pixel_embedding_for_grouping_poster.pdf}{Poster}]
}

\maketitle

\input{section01_abstract}
\input{section02_intro4arxiv}

\input{section03_learning_embedding_space4arxiv}

\input{section04_meanshift_grouping4arxiv}

\input{section05_experiment4arxiv}

\input{section06_conclusion}
\input{section07_acknowledgement}

{\small
\bibliographystyle{ieee}
\bibliography{egbib}
}

\clearpage\mbox{}

\begin{center}
 {\large \textbf{Appendix}}
\end{center}

In this appendix, we provide proofs of the propositions
introduced in the main paper for understanding our objective function and
grouping mechanism.  Then, we provide the details of the mean-shift algorithm,
computation of gradients and how it is adapted for recurrent grouping.  We
illustrate how the gradients are back-propagated to the input embedding using a
toy example.
Finally, we include more qualitative results on boundary detection and instance
segmentation.

\input{section07_appendix4arxiv}

\end{document}

%% file: section01_abstract.tex
\begin{abstract}
We introduce a differentiable, end-to-end trainable framework for solving
pixel-level grouping problems such as instance segmentation consisting of two
novel components. First, we regress pixels into a hyper-spherical embedding
space so that pixels from the same group have high cosine similarity while
those from different groups have similarity below a specified margin.  We
analyze the choice of embedding dimension and margin, relating them to
theoretical results on the problem of distributing points uniformly on the
sphere. Second, to group instances, we utilize a variant of mean-shift
clustering, implemented as a recurrent neural network parameterized by kernel
bandwidth.  This recurrent grouping module is differentiable, enjoys convergent
dynamics and probabilistic interpretability.  Backpropagating the
group-weighted loss through this module allows learning to focus on only
correcting embedding errors that won't be resolved during subsequent
clustering.  Our framework, while conceptually simple and theoretically
abundant, is also practically effective and computationally efficient. We
demonstrate substantial improvements over state-of-the-art instance
segmentation for object proposal generation, as well as demonstrating
the benefits of grouping loss for classification tasks such as boundary
detection and semantic segmentation.


\end{abstract}


%% file: section02_intro4arxiv.tex
\section{Introduction}

\begin{figure}[t]
\centering
   \includegraphics[width=0.995\linewidth]{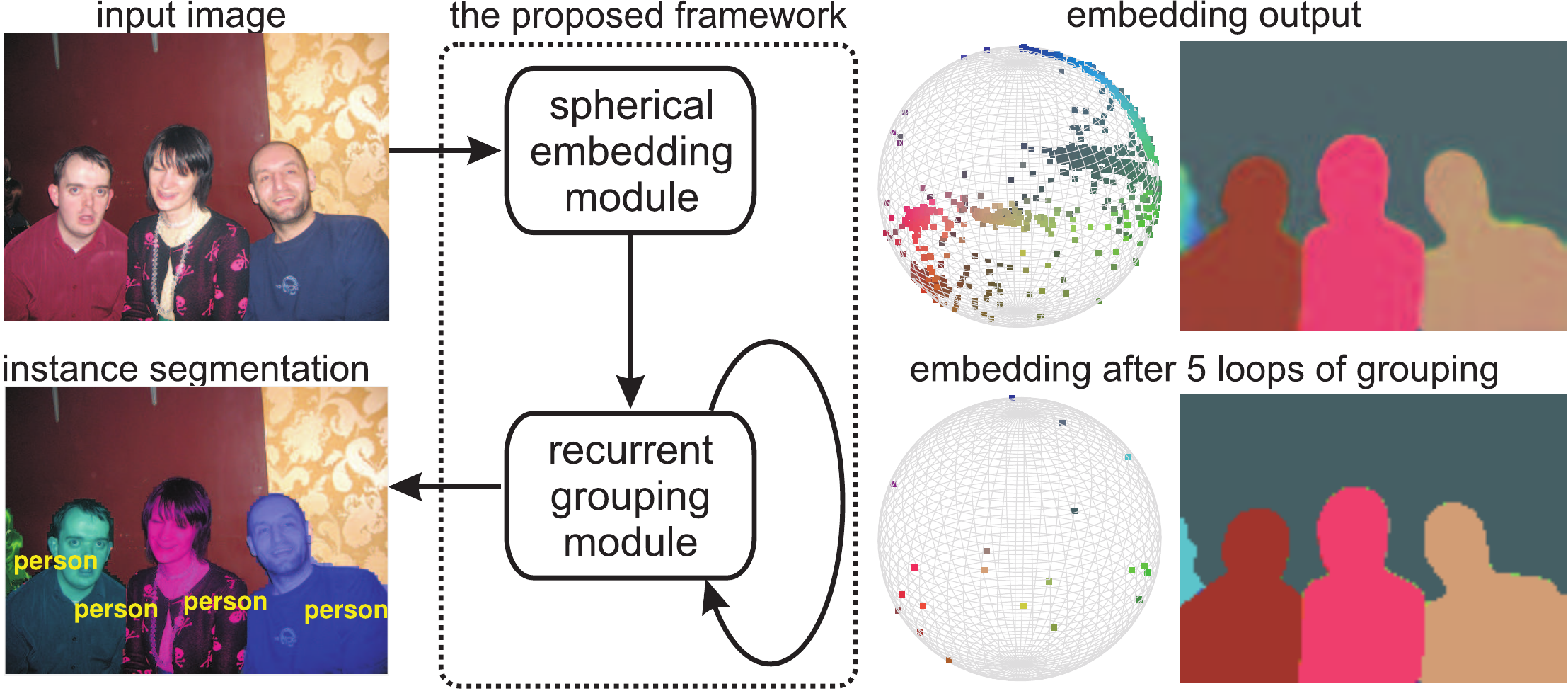}
   \vspace{-3.5mm}
   \caption{
   \small Our framework embeds pixels into a hyper-sphere where
   recurrent mean-shift dynamics groups pixels
   into a variable number of object instances.  Here we visualize random
   projections of a 64-dim embeddings into 3-dimensions.
   }
   \vspace{-5.mm}
\label{fig:splash_figure}
\end{figure}

The successes of deep convolutional neural nets (CNNs) at image classification
has spawned a flurry of work in computer vision on adapting these models to
pixel-level image understanding tasks, such as boundary detection
\cite{arbelaez2011contour, xie2015holistically, maninis2017convolutional},
semantic segmentation \cite{long2015fully,chen2016deeplab,kong2017recurrent},
optical flow \cite{weinzaepfel2013deepflow, dosovitskiy2015flownet}, and pose
estimation \cite{wei2016cpm, cao2017realtime}.  The key ideas that have enabled
this adaption thus far are: (1) deconvolution schemes that allow for upsampling
coarse pooled feature maps to make detailed predictions at the spatial
resolution of individual pixels~\cite{xie2015holistically,ghiasi2016laplacian},
(2) skip connections and hyper-columns which concatenate representations across
multi-resolution feature maps~\cite{hariharan2015hypercolumns,chen2016deeplab},
(3) atrous convolution which allows efficient computation with large receptive
fields while maintaining spatial resolution~\cite{chen2016deeplab,
kong2017recurrent}, and (4) fully convolutional operation which handles
variable sized input images.

In contrast, there has been less innovation in the development of specialized
loss functions for training. Pixel-level labeling tasks fall into the category
of structured output prediction~\cite{bakir2007predicting}, where the model
outputs a structured object (e.g., a whole image parse) rather than a scalar or
categorical variable.  However, most CNN pixel-labeling architectures are
simply trained with loss functions that decompose into a simple (weighted) sum
of classification or regression losses over individual pixel labels.

The need to address the output space structure is more apparent when
considering problems where the set of output labels isn't fixed.  Our
motivating example is object instance segmentation, where the model generates a
collection of segments corresponding to object instances.  This problem
can't be treated as k-way classification since the number of objects isn't
known in advance.  Further, the loss should be invariant to permutations of
the instance labels within the same semantic category.

As a result, most recent successful approaches to instance segmentation have
adopted more heuristic approaches that first use an object detector to
enumerate candidate instances and then perform pixel-level segmentation of each
instance~\cite{liang2015proposal, dai2016instance, li2016fully,
liang2016reversible, arnab2017pixelwise}.  Alternately one can generate generic
proposal segments and then label each one with a semantic
detector~\cite{hariharan2014simultaneous, chen2015multi,
hariharan2015hypercolumns, dai2015convolutional, uhrig2016pixel, he2017mask}.
In either case the detection and segmentation steps can both be mapped to
standard binary classification losses.  While effective, these approaches are
somewhat unsatisfying since: (1) they rely on the object detector and
non-maximum suppression heuristics to accurately ``count'' the number of
instances, (2) they are difficult to train in an end-to-end manner since the
interface between instance segmentation and detection is non-differentiable,
and (3) they underperform in cluttered scenes as the assignment of pixels to
detections is carried out independently for each detection\footnote{This is
less a problem for object proposals that are jointly estimated by bottom-up
segmentation (e.g., MCG~\cite{pont2017multiscale} and COB
\cite{maninis2017convolutional}).  However, such generic proposal generation is
not informed by the top-down semantics.}.

Here we propose to directly tackle the instance grouping problem in a unified
architecture by training a model that labels pixels with unit-length vectors
that live in some fixed-dimension embedding space
(Fig.~\ref{fig:splash_figure}).  Unlike k-way classification where the target
vectors for each pixel are specified in advance (i.e., one-hot vectors at the
vertices of a k-1 dimensional simplex) we allow each instance to be labeled
with an arbitrary embedding vector on the sphere.  Our loss function simply
enforces the constraint that the embedding vectors used to label different
instances are far apart.  Since neither the number of labels, nor the target
label vectors are specified in advance, we can't use standard soft-max
thresholding to produce a discrete labeling.  Instead, we utilize a
variant of mean-shift clustering which can be viewed as a recurrent network
whose fixed point identifies a small, discrete set of instance label vectors
and concurrently labels each pixel with one of the vectors from this set.

This framework is largely agnostic to the underlying CNN architecture and can
be applied to a range of low, mid and high level visual tasks. Specifically,
we carry out experiments showing how this method can be used for boundary
detection, object proposal generation and semantic instance segmentation. Even
when a task can be modeled by a binary pixel classification loss (e.g.,
boundary detection) we find that the grouping loss guides the model towards
higher-quality feature representations that yield superior performance to
classification loss alone. The model really shines for instance segmentation,
where we demonstrate a substantial boost in object proposal generation
(improving the state-of-the-art average recall for 10 proposals per image from
0.56 to 0.77).  To summarize our contributions: (1) we introduce a simple,
easily interpreted end-to-end model for pixel-level instance labeling which is
widely applicable and highly effective, (2) we provide theoretical analysis
that offers guidelines on setting hyperparameters, and (3) benchmark results
show substantial improvements over existing approaches.

\section{Related Work}

Common approaches to instance segmentation first generate region proposals or
class-agnostic bounding boxes, segment the foreground objects within each
proposal and classify the objects in the bounding box~\cite{yang2012layered,
ladicky2010and, hariharan2014simultaneous, chen2015multi, dai2016instance,
liang2016reversible, he2017mask}.  \cite{li2016fully} introduce a fully
convolutional approach that includes bounding box proposal generation in
end-to-end training.  Recently, ``box-free''
methods~\cite{pinheiro2015learning, pinheiro2016learning, liang2015proposal,
hu2016fastmask} avoid some limitations of box proposals (e.g. for wiry or
articulated objects). They commonly use Faster RCNN~\cite{ren2015faster} to
produce ``centeredness'' score on each pixel and then predict binary instance
masks and class labels.  Other approaches have been explored for modeling joint
segmentation and instance labeling jointly in a combinatorial framework (e.g.,
\cite{kirillov2016instancecut}) but typically don't address end-to-end
learning.  Alternately, recurrent models that sequentially produce a list of
instances~\cite{romera2016recurrent,renend} offer another approach to address
variable sized output structures in a unified manner.

The most closely related to ours is the associative embedding work of
\cite{newell2016associative}, which demonstrated strong results for grouping
multi-person keypoints, and unpublished work from \cite{fathi2017semantic} on
metric learning for instance segmentation.  Our approach extends on these ideas
substantially by integrating recurrent mean-shift to directly generate the
final instances (rather than heuristic decoding or thresholding distance to
seed proposals). There is also an important and interesting connection to work
that has used embedding to separate instances where the embedding is directly
learned using a supervised regression loss rather than a pairwise associative
loss.  \cite{sironi2014multiscale} train a regressor that predicts the distance
to the contour centerline for boundary detection, while \cite{bai2016deep}
predict the distance transform of the instance masks which is then
post-processed with watershed transform to generate segments.
\cite{uhrig2016pixel} predict an embedding based on scene depth and direction
towards the instance center (like Hough voting).

Finally, we note that these ideas are related to work on using embedding for
solving pairwise clustering problems. For example, normalized cuts clusters
embedding vectors given by the eigenvectors of the normalized graph
Laplacian~\cite{shi2000normalized} and the spatial gradient of these embedding
vectors was used in~\cite{arbelaez2011contour} as a feature for boundary
detection. Rather than learning pairwise similarity from data and then
embedding prior to clustering (e.g., \cite{maire2016affinity}), we use a
pairwise loss but learn the embedding directly. Our recurrent mean-shift
grouping is reminiscent of other efforts that use unrolled implementations of
iterative algorithms such as CRF inference~\cite{zheng2015conditional} or
bilateral filtering~\cite{jampani2016learning, gadde2015superpixel}. Unlike
general RNNs~\cite{bengio1994learning, pascanu2013difficulty} which are often
difficult to train, our recurrent model has fixed parameters that assure
interpretable convergent dynamics and meaningful gradients during learning.

%
%
%
%
%
%
%
%
%

%% file: section03_learning_embedding_space4arxiv.tex
\section{Pairwise Loss for Pixel Embeddings}
\label{sec:max-margin}
In this section we introduce and analyze the loss we use for learning pixel
embeddings. This problem is broadly related to supervised distance metric
learning~\cite{weinberger2009distance,kong2012dictionary, kong2013learning} and clustering~\cite{kong2012multi}
but adapted to the specifics of instance labeling where
the embedding vectors are treated as labels for a variable number of objects in
each image.

Our goal is to learn a mapping from an input image to a set of $D$-dimensional
embedding vectors (one for each pixel).  Let $\x_i,\x_j \in {\mathbb R}^D$ be the
embeddings of pixels $i$ and $j$ respectively with corresponding labels $y_i$
and $y_j$ that denote ground-truth instance-level semantic labels (e.g., {\em
car.1} and {\em car.2}). We will measure the similarity of the embedding vectors
using the cosine similarity, been scaled and offset to lie in the interval
$[0,1]$ for notational convenience:
\begin{equation}
\begin{split} \small
s_{ij} = \frac{1}{2}\left(1 + \frac{\x_i^T\x_j}{\Vert\x_i\Vert_2 \Vert\x_j\Vert_2}\right)
\end{split}
\label{eq:calibrated_cosine_similarity}
\end{equation}
In the discussion that follows we think of the similarity in terms of the inner
product between the projected embedding vectors (e.g., $\frac{x_i}{\|x_i\|}$)
which live on the surface of a $(D-1)$ dimensional sphere.  Other common
similarity metrics utilize Euclidean distance with a squared exponential kernel
or sigmoid function~\cite{newell2016associative, fathi2017semantic}. We prefer
the cosine metric since it is invariant to the scale of the embedding vectors,
decoupling the loss from model design choices such as weight decay or
regularization that limit the dynamic range of Euclidean distances.

Our goal is to learn an embedding so that pixels with the same label (positive
pairs with $y_i = y_j$) have the same embedding (i.e. $s_{ij}=1$). To avoid a
trivial solution where all the embedding vectors are the same, we impose the
additional constraint that pairs from different instances (negative pairs with
$y_i \neq y_j$) are placed far apart. To provide additional flexibility, we
include a weight $w_i$ in the definition of the loss which specifies the
importance of a given pixel. The total loss over all pairs and training images
is:
\begin{equation}
\small
\begin{split}
\ell = \sum_{k=1}^M  \sum_{i,j=1}^{N_k} \frac{w^k_i w^k_j}{N_k} \Big( \1_{\{y_i=y_j\}}(1-s_{ij})
 &+ \1_{\{y_i\not=y_j\}} [s_{ij}-\alpha]_{+} \Big)
\end{split}
\label{eq:obj}
\end{equation}
where $N_k$ is the number of pixels in the $k$-th image ($M$ images in
total), and $w^k_i$ is the pixel pair weight associated with pixel
$i$ in image $k$.  The hyper-parameter $\alpha$ controls the maximum margin
for negative pairs of pixels, incurring a penalty if the embeddings for pixels
belonging to the same group have an angular separation of less than
$\cos^{-1}(\alpha)$.  Positive pairs pay a penalty if they have a similarity
less than $1$.  Fig.~\ref{fig:lossFunc} shows a graph of the loss function.
\cite{wu2017sampling} argue that the constant slope of the margin loss is more
robust, e.g., than squared loss.

\begin{figure}[t]
\vspace{-1mm}
\raisebox{-\height}{
\includegraphics[width=0.20\textwidth]{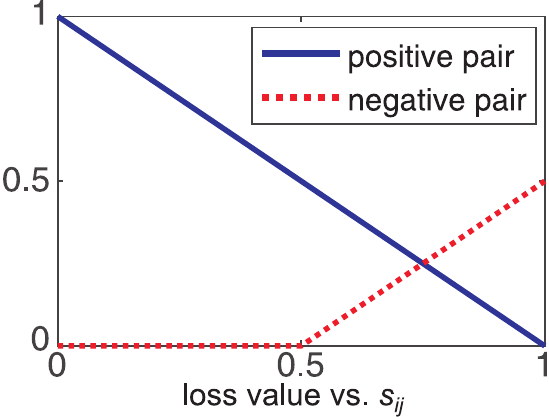}}\hfill%
\begin{minipage}[t]{0.23\textwidth} 
\caption{\small
Loss as a function of calibrated similarity score Eq.~\ref{eq:calibrated_cosine_similarity}
with $\alpha=0.5$. The
gradient is constant, limiting the effect of noisy ground-truth labels
 (i.e., near an object boundary)}
\label{fig:lossFunc}
\end{minipage}
\vspace{-4mm}
\end{figure}

We carry out a simple theoretical analysis which provides a guide for setting
the weights $w_i$ and margin hyperparameter $\alpha$ in the loss function. Proofs
can be found in the appendix.

\subsection{Instance-aware Pixel Weighting}
We first examine the role of embedding dimension and instance size on the
training loss.
\begin{propositions}
\label{theorem:ObjLowerBound}
For $n$ vectors $\{\x_1, \dots, \x_n\}$, the total intra-pixel similarity is
bounded as $\sum_{i\neq j} \x_i^T\x_j \ge -\sum_{i=1}^n  \Vert \x_i \Vert_2^2$.
In particular, for $n$ vectors on the hypersphere where $\Vert\x_i\Vert_2 = 1$,
we have $\sum_{i\neq j} \x_i^T\x_j \ge -n$.
\end{propositions}
This proposition indicates that the total cosine similarity (and hence the loss)
for a set of embedding vectors has a constant lower bound that does not depend
on the dimension of the embedding space (a feature lacking in Euclidean
embeddings). In particular, this type of analysis suggests a natural choice of
pixel weighting $w_i$. Suppose a training example contains $Q$ instances and
${\cal I}_q$ denotes the set of pixels belonging to a particular ground-truth
instance $q$.  We can write
\begin{equation}
\small
\begin{split}
\| \sum_{q=1}^Q \sum_{i \in {\cal I}_q} w_i \x_i \|^2 = &
\sum_{q=1}^Q \| \sum_{i \in {\cal I}_q} w_i \x_i \|^2 + \\
&\sum_{p \neq q} \big(\sum_{i \in {\cal I}_p} w_i \x_i \big)^T \big(\sum_{j \in {\cal I}_q} w_j \x_j \big)
\end{split}
\nonumber
\end{equation}
where the first term on the r.h.s. corresponds to contributions to the loss
function for positive pairs while the second corresponds to contributions
from negative pairs. Setting $w_i = \frac{1}{\vert {\cal I}_q\vert}$ for pixels $i$
belonging to ground-truth instance $q$ assures that each instance contributes
equally to the loss independent of size. Furthermore, when the embedding
dimension $D \geq Q$, we can simply embed the data so that the instance means
$ \mu_k = \frac{1}{\vert {\cal I}_q\vert} \sum_{i \in {\cal I}_q} \x_i $
are along orthogonal axes on the sphere. This zeros out the second term on
the r.h.s., leaving only the first term which is bounded
$
0 \leq \sum_{q=1}^Q \left\| \frac{1}{\vert {\cal I}_q \vert} \sum_{i \in {\cal I}_q} \x_i \right\|^2 \leq Q$,
and translates to corresponding upper and lower bounds on the loss
that are independent of the number of pixels and embedding dimension
(so long as $D \geq Q$).

Pairwise weighting schemes have been shown important empirically
\cite{fathi2017semantic} and class imbalance can have a substantial effect on
the performance of different architectures (see e.g., \cite{lin2017focal}).
While other work has advocated online bootstrapping methods for hard-pixel
mining or mini-batch selection ~\cite{loshchilov2015online, kong2016photo,
shrivastava2016training,wu2016bridging},
our approach is much simpler. Guided by
this result we simply use uniform random sampling of pixels during training,
appropriately weighted by instance size in order to estimate the loss.

\subsection{Margin Selection}
To analyze the appropriate margin, let's first consider the problem of
distributing labels for different instances as far apart as possible on a 3D
sphere, sometimes referred to as Tammes's problem, or the hard-spheres
problem~\cite{saff1997distributing}. This can be formalized as maximizing
the smallest distance among $n$ points on a sphere: $\max\limits_{\x_i\in
\RB^3} \min\limits_{i\not=j}\Vert \x_i-\x_j\Vert_2$.  Asymptotic results
in~\cite{habicht1951lagerung} provide the following proposition (see proof in the
appendix):
\begin{propositions}
\label{lemma:max_margin}
Given $N$ vectors $\{\x_1, \dots, \x_n\}$ on a 2-sphere, i.e. $\x_i \in \RB^3$,
$\Vert\x_i\Vert_2 = 1, \forall i=1\dots n$, choosing $\alpha \leq 1- \Big(
\frac{2\pi}{\sqrt{3}N} \Big)$, guarantees that $[s_{ij}-\alpha]_{+}\geq 0$ for
some pair $i \not= j$. Choosing $\alpha > 1- \frac{1}{4} \Bigg(
\Big(\frac{8\pi}{\sqrt{3}N}\Big)^{\frac{1}{2}} -CN^{-\frac{2}{3}}\Bigg)^2$,
guarantees the existence of an embedding with $[s_{ij}-\alpha]_{+}=0$ for all
pairs $i \not= j$.
\end{propositions}
Proposition~\ref{lemma:max_margin} gives the maximum margin for a separation of $n$
groups of pixels in a three dimensional embedding space (sphere).  For example,
if an image has at most $\{4,5,6,7\}$ instances, $\alpha$ can be set as small
as $\{0.093, 0.274, 0.395, 0.482\}$, respectively.

For points in a higher dimension embedding space, it is a non-trivial problem to
establish a tight analytic bound for the margin $\alpha$. Despite its simple
description, distributing $n$ points on a $(D-1)$-dimensional hypersphere is
considered a serious mathematical challenge for which there is no general
solutions~\cite{saff1997distributing, lovisolo2001uniform}. We adopt a safe
(trivial) strategy. For $n$ instances embedded in $n/2$ dimensions one can use
value of $\alpha=0.5$ which allows for zero loss by placing a pair of groups
antipodally along each of the $n/2$ orthogonal axes. We adopt this setting for
the majority of experiments in the paper where the embedding dimension is set
to $64$.

%% file: section04_meanshift_grouping4arxiv.tex
\section{Recurrent Mean-Shift Grouping}
\label{sec:meanShiftGrouping}

\begin{figure}[t]
\centering
   \includegraphics[width=1\linewidth]{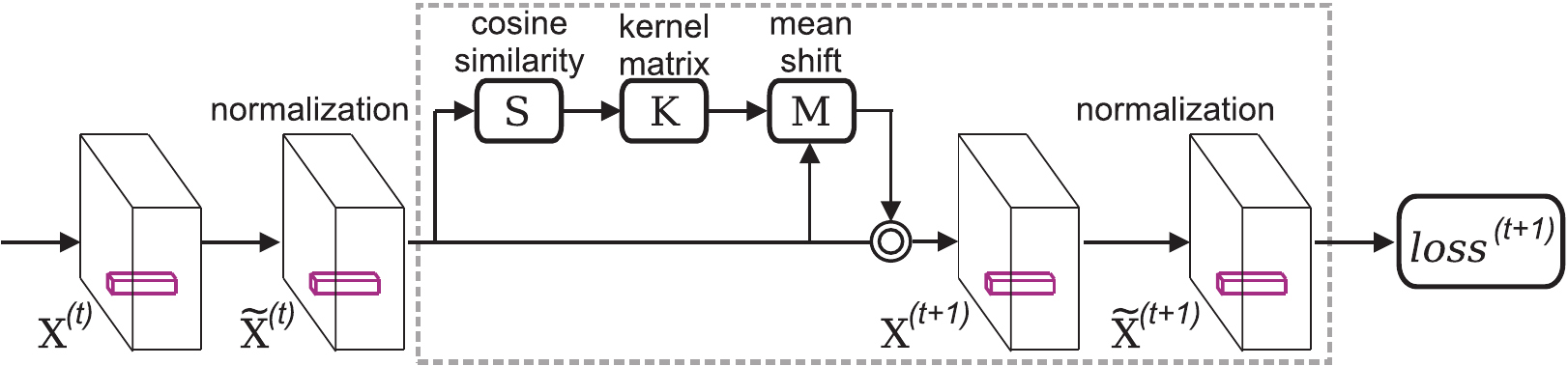} 
   \vspace{-6mm}
   \caption{Recurrent mean shift grouping module is unrolled during training.
   }
   \vspace{-3mm}
\label{fig:meanShiftGroupingModule}
\end{figure}

While we can directly train a model to predict embeddings as
described in the previous section, it is not clear how to generate the final
instance segmentation from the resulting (imperfect) embeddings. One can
utilize heuristic post-processing~\cite{de2017semantic} or utilize clustering
algorithms that estimate the number of instances~\cite{liang2015proposal},
but
these are not differentiable and thus unsatisfying. Instead, we introduce a
mean-shift grouping model (Fig.~\ref{fig:meanShiftGroupingModule})
which operates recurrently on the embedding space in order to congeal
the embedding vectors into a small number of instance labels.

Mean-shift and closely related algorithms~\cite{fukunaga1975estimation,
cheng1995mean, comaniciu1999mean, comaniciu2002mean} use kernel density
estimation to approximate the probability density from a set of samples and
then perform clustering on the input data by assigning or moving each sample to
the nearest mode (local maxima). From our perspective, the advantages of this
approach are (1) the final instance labels (modes) live in the same embedding
space as the initial data, (2) the recurrent dynamics of the clustering process
depend smoothing on the input allowing for easy backpropagation, (3) the
behavior depends on a single parameter, the kernel bandwidth, which is easily
interpretable and can be related to the margin used for the embedding loss.

\subsection{Mean Shift Clustering}
A common choice for non-parametric density estimation is to use the isotropic
multivariate normal kernel $K(\x, \x_i)= (2\pi)^{-D/2}\exp \Big(-
\frac{\delta^2}{2} \Vert\x - \x_i \Vert^2_2 \Big)$ and approximate the data
density non-parametrically as $p(x) = \frac{1}{N} \sum K(x,x_i)$. Since our
embedding vectors are unit norm, we instead use the von Mises-Fisher
distribution which is the natural extension of the multivariate normal to the
hypersphere~\cite{fisher1953dispersion,
banerjee2005clustering,mardia2009directional,kobayashi2010mises},
and is given by $K(\x,\x_i) \propto \exp( \delta \x^T \x_i )$.
The kernel bandwidth, $\delta$ determines the smoothness of the
kernel density estimate and is closely related to the margin used for learning
the embedding space. While it is straightforward to learn $\delta$ during
training, we instead set it to satisfy $\frac{1}{\delta}=\frac{1-\alpha}{3}$
throughout our experiments, such that the cluster separation (margin) in the
learned embedding space is three standard deviations.

We formulate the mean shift algorithm in a matrix form.  Let $\X \in
\RB^{D\times N}$ denote the stacked $N$ pixel embedding vectors of an image.
The kernel matrix is given by $\K = \exp(\delta \X^T\X) \in \RB^{N \times N}$.
Let $\D = \diag(\K^T\1)$ denote the diagonal matrix of total affinities,
referred to as the degree when $\K$ is viewed as a weighted graph
adjacency matrix. At each iteration, we compute the mean shift $\M =
\X\K\D^{-1} - \X$, which is the difference vector between $\X$ and the kernel
weighted average of $\X$. We then modify the embedding vectors by moving
them in the mean shift direction with step size $\eta$:
\begin{equation}
\begin{split}
\X \leftarrow & \X + \eta(\X\K\D^{-1} - \X)\\
\leftarrow & \X (\eta\K\D^{-1} + (1-\eta)\I)\\
\end{split}
\end{equation}
Note that unlike standard mean-shift mode finding, we recompute $\K$ at each
iteration.  These update dynamics are termed the explicit-$\eta$ method and
were analyzed by~\cite{carreira2008generalised}. When $\eta=1$ and the kernel
is Gaussian, this is also referred to as Gaussian Blurring Mean Shift (GBMS)
and has been shown to have cubic convergence~\cite{carreira2008generalised}
under appropriate conditions.  Unlike deep RNNs, the parameters of our
recurrent module are not learned and the forward dynamics are convergent under
general conditions.  In practice, we do not observe issues with exploding or
vanishing gradients during back-propagation through a finite number of
iterations~\footnote{Some intuition about stability may be gained by noting
that the eigenvalues of $\K\D^{-1}$ lie in the interval $[0,1]$, but we have
not been able to prove useful corresponding bounds on the spectrum of the
Jacobian.}.

Fig.~\ref{fig:mnist_demo} demonstrates a toy example of applying the method to
perform digit instance segmentation on synthetic images from
MNIST~\cite{lecun1998gradient}. We learn 3-dimensional embedding in order to
visualize the results before and after the mean shift grouping module.  From
the figure, we can see the mean shift grouping transforms the initial embedding
vectors to yield a small set of instance labels which are distinct (for negative
pairs) and compact (for positive pairs).

\begin{figure}[t]
\centering
   \includegraphics[width=0.95\linewidth]{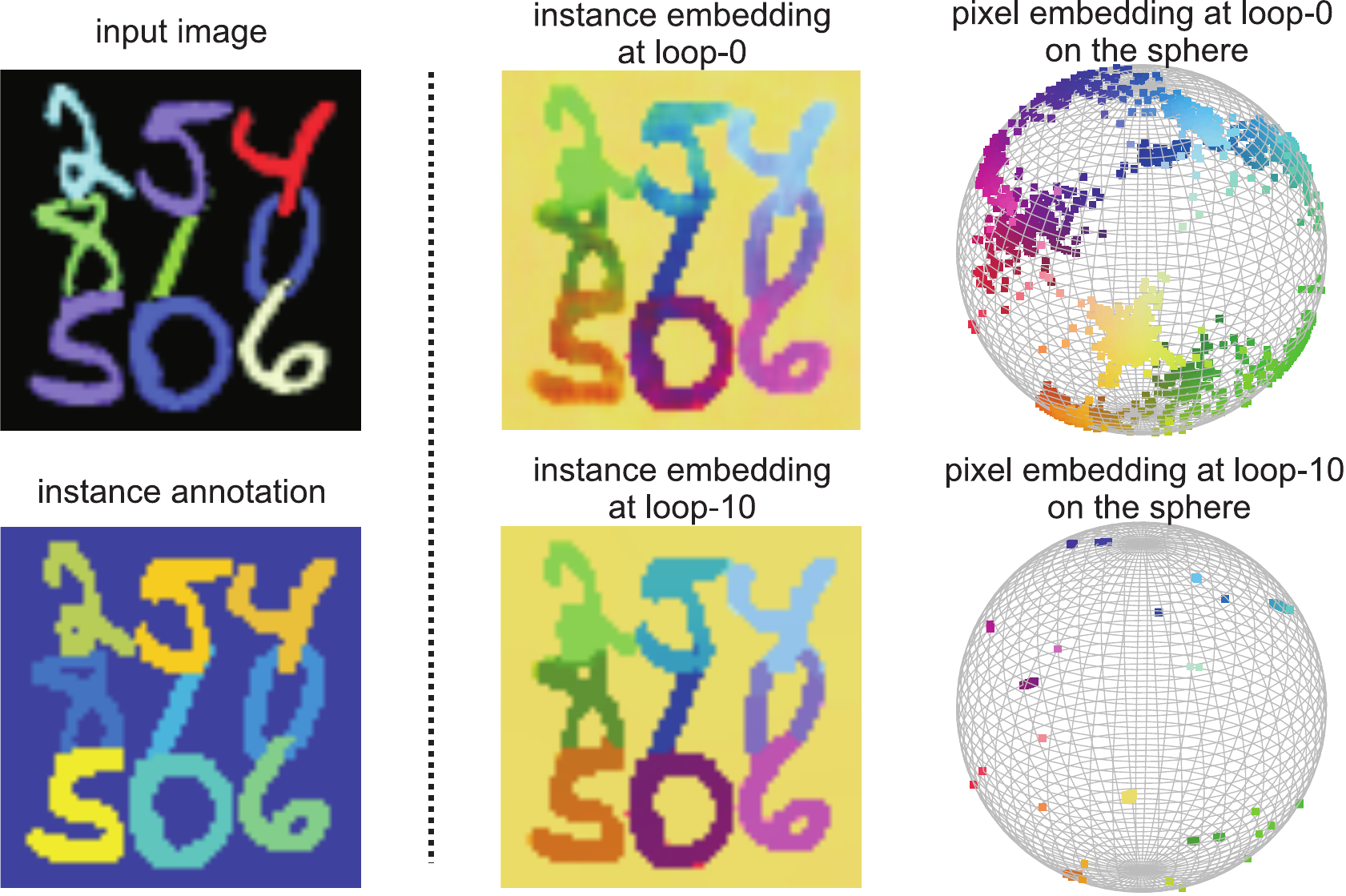}
   \vspace{-2mm}
   \caption{Demonstration of mean-shift grouping on a synthetic image and with
   ground-truth instance identities (left panel).  Right panel: the pixel
   embedding visualization at 3-dimensional embedding sphere (upper row) and
   after 10 iterations of recurrent mean-shift grouping (bottom row).
   }
   \vspace{-3mm}
\label{fig:mnist_demo}
\end{figure}

\subsection{End-to-end training}
It's straightforward to compute the derivatives of the recurrent mean shift
grouping module w.r.t $\X$ based on the the chain rule so our whole system is
end-to-end trainable through back-propagation.  Details about the derivative
computation can be found in the appendix.  To understand the benefit
of end-to-end training, we visualize the embedding gradient with and without
the grouping module (Fig.~\ref{fig:analysis_show_on_paper}).  Interestingly, we
observe that the gradient backpropagated through mean shift focuses on fixing
the embedding in uncertain regions, e.g.  instance boundaries, while suggesting
small magnitude updates for those errors which will be easily fixed by the
mean-shift iteration.

While we could simply apply the pairwise embedding loss to the final output of
the mean-shift grouping, in practice we accumulate the loss over all iterations
(including the initial embedding regression).
We unroll the recurrent grouping module into $T$ loops,
and accumulate the same loss function at the unrolled loop-$t$:
\begin{equation}
\small
\begin{split}
\ell^{t} =&  \sum_{k=1}^M  \sum_{i,j \in S_k} \frac{w^k_i w^k_j}{\vert S_k \vert}  \Big( \1_{\{y_i=y_j\}} ( 1 - s_{ij}^{t})   + \1_{\{y_i\not=y_j\}} [s^{t}_{ij}-\alpha]_{+} \Big) \\
\ell =& \sum_{t=1}^T \ell^{t}
\end{split}
\nonumber
\end{equation}

\begin{figure}[t]
\centering
   \includegraphics[width=1\linewidth]{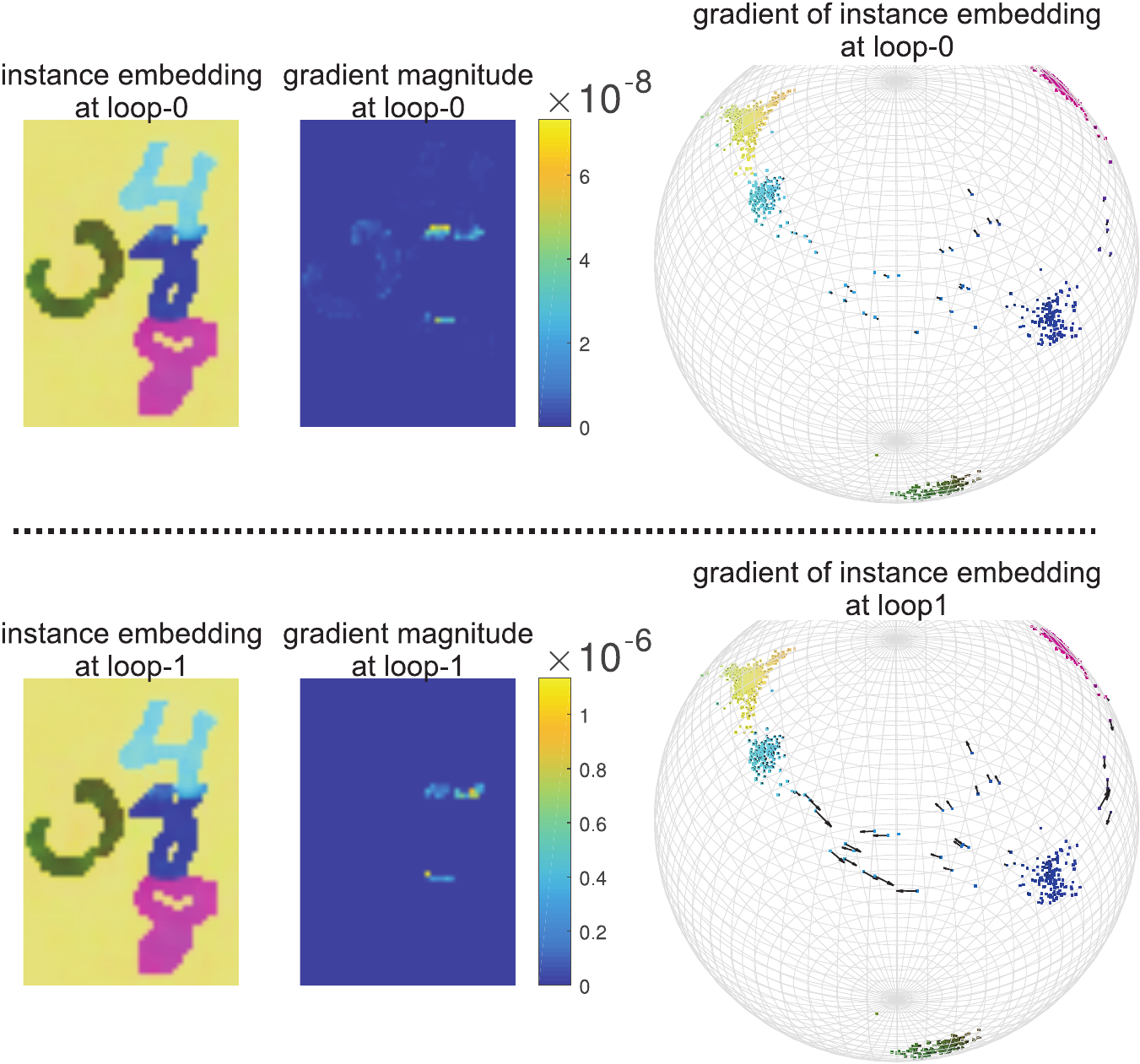}
   \vspace{-6mm}
   \caption{To analyze the recurrent mean shift grouping module, we compare the
   embedding vector gradients with and without one loop of grouping.  The
   length of arrows in the projection demonstrates the gradient magnitude,
   which are also depicted in maps as the second column. Backpropagating the
   loss through the grouping module serves to focus updates on embeddings of
   ambiguous pixels near boundaries while ignoring pixels with small errors
   which will be corrected by the subsequent grouping process.
   }
   \vspace{-2mm}
\label{fig:analysis_show_on_paper}
\end{figure}

%% file: section05_experiment4arxiv.tex
\section{Experiments}

We now describe experiments in training our framework to deal a
variety of pixel-labeling problems, including boundary detection, object
proposal detection, semantic segmentation and instance-level semantic
segmentation.

\subsection{Tasks, Datasets and Implementation}

We illustrate the advantages of the proposed modules on several large-scale
datasets.  First, to illustrate the ability of the instance-aware weighting and
uniform sampling mechanism to handle imbalanced data and low embedding
dimension, we use the BSDS500~\cite{arbelaez2011contour} dataset to train a
boundary detector for boundary detection ($>90\%$ pixels are
non-boundary pixels).  We train with the standard
split~\cite{arbelaez2011contour,xie2015holistically},
using 300 train-val images to train our model based on
ResNet50~\cite{he2016deep} and evaluate on the remaining 200 test images.
Second, to explore instance segmentation and object proposal generation, we use
PASCAL VOC 2012 dataset~\cite{everingham2010pascal} with additional instance
mask annotations provided by \cite{hariharan2011semantic}. This provides 10,582
and 1,449 images for training and evaluation, respectively.

We implement our approach using the toolbox
MatConvNet~\cite{vedaldi2015matconvnet}, and train using SGD on a single Titan
X GPU.
\footnote{The code and trained models can be found at
{\color{blue} \emph{
{https://github.com/aimerykong/Recurrent-Pixel-Embedding-for-Instance-Grouping}}}}.
To compute calibrated cosine similarity, we utilize an L2-normalization
layer before matrix multiplication~\cite{kong2016low}, which also contains
random sampling with a hyper-parameter to control the ratio of pixels to be
sampled for an image.  In practice, we observe that performance does not depend
strongly on this ratio and hence set it based on available (GPU) memory.

While our modules are architecture agnostic, we use the
ResNet50 and ResNet101 models~\cite{he2016deep} pre-trained over
ImageNet~\cite{deng2009imagenet} as the backbone.
Similar to~\cite{chen2016deeplab}, we increase the output resolution of ResNet
by removing the top global $7\times 7$ pooling layer and the last two $2\times2$
pooling layers, replacing them with atrous convolution with dilation rate 2 and
4, respectively to maintain a spatial sampling rate. Our model thus outputs
predictions at $1/8$ the input resolution which are upsampled for benchmarking.

We augment the training set using random scaling by $s\in [0.5, 1.5]$, in-plane
rotation by $[-10^\circ,10^\circ]$ degrees, random left-right flips, random
crops with 20-pixel margin and of size divisible by 8, and color jittering.
When training the model, we fix the batch normalization in ResNet backbone,
using the same constant global moments in both training and testing.
Throughout training, we set batch size to one where the batch is a single input
image.
We use the ``poly'' learning rate policy~\cite{chen2016deeplab} with a base
learning rate of $2.5e-4$ scaled as a function of iteration by
$(1-\frac{iter}{maxiter})^{0.9}$.

\subsection{Boundary Detection}

\begin{figure}[t]
\centering
   \includegraphics[width=0.98\linewidth]{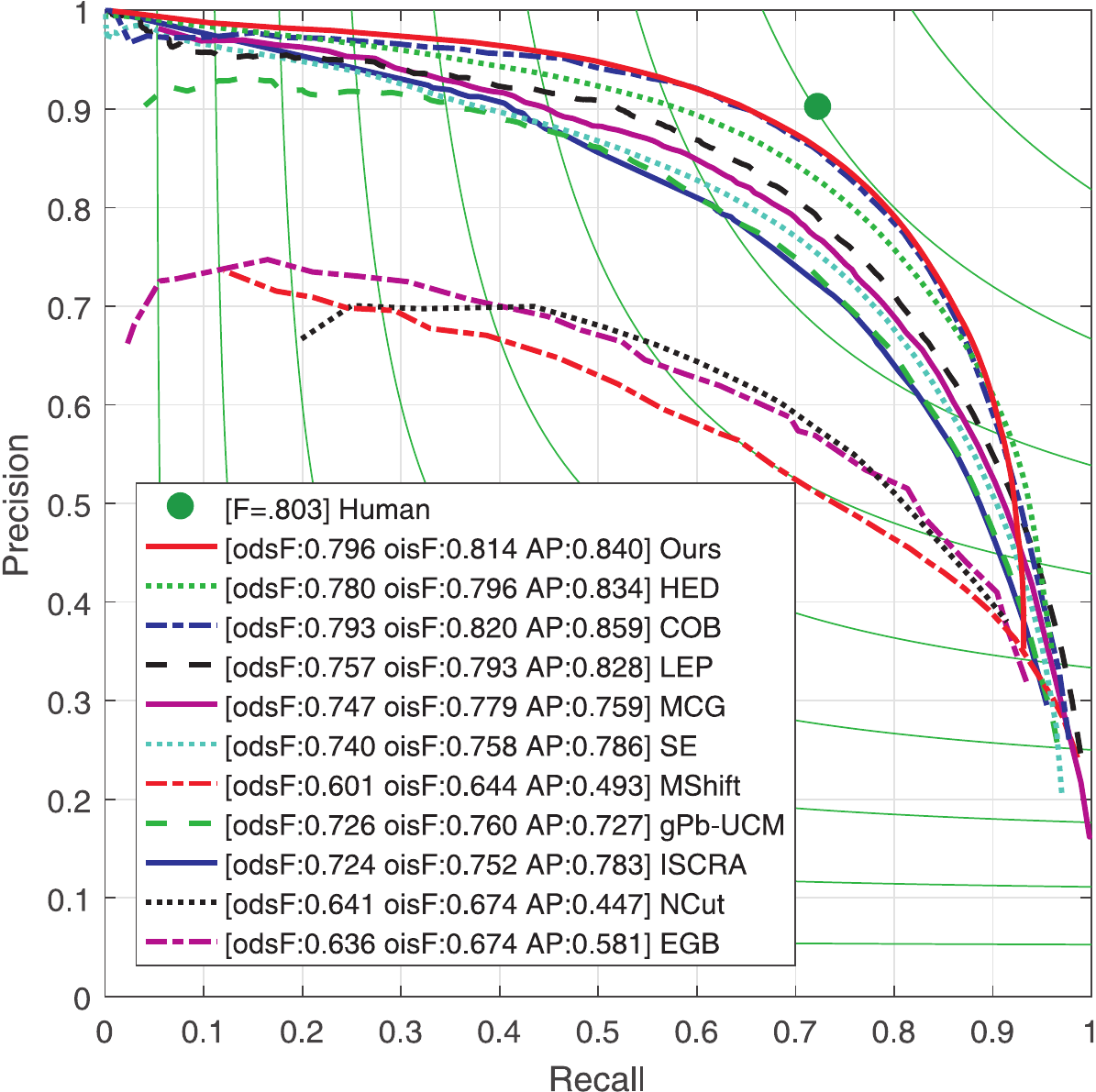}
   \vspace{-2mm}
   \caption{Boundary detection performance on BSDS500}
\label{fig:boundary_prcurve}
\vspace{-3mm}
\end{figure}

\begin{figure}[t]
\centering
   \includegraphics[width=1\linewidth]{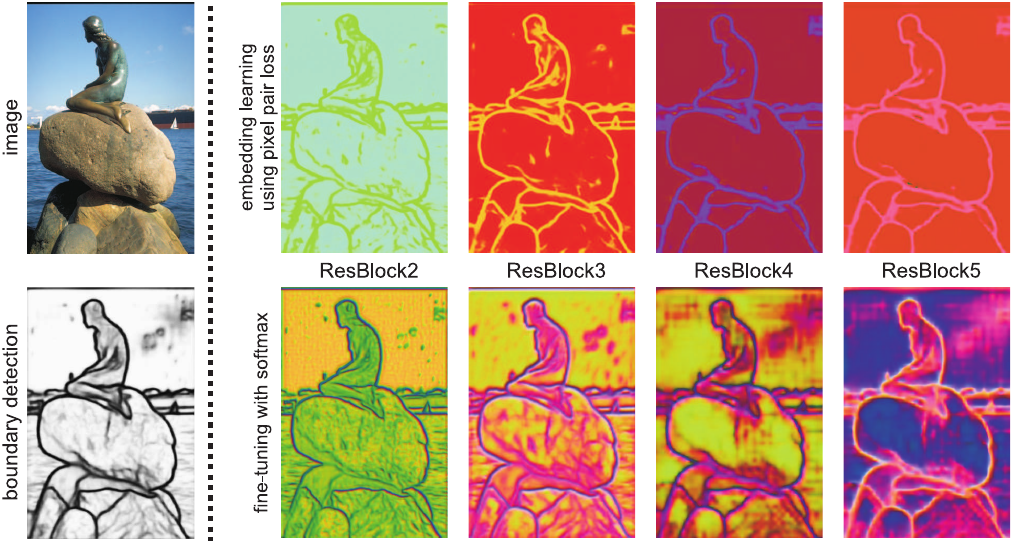}
\vspace{-5mm}
   \caption{
   Visualization of boundary detection embeddings.  We show the 3D embedding as
   RGB images (more examples in appendix).  The upper and lower row in the
   right panel show embedding vectors at different layers from the model before
   and after fine-tuning using logistic loss.  After fine-tuning not only
   predict the boundary pixels, but also encode boundary orientation and signed
   distance to the boundary, similar to supervised embedding approaches
   \cite{sironi2014multiscale,uhrig2016pixel,bai2016deep}
   }
\label{fig:boundary_show_in_paper}
\vspace{-2mm}
\end{figure}

For boundary detection, we first train a model to group the pixels into
boundary or non-boundary groups.  Similar to
COB~\cite{maninis2017convolutional} and HED~\cite{xie2015holistically}, we
include multiple branches over ResBlock $2, 3, 4, 5$ for training.
Since the number of instances labels is 2, we learn a simple 3-dimensional
embedding space which has the advantage of easy visualization as an RGB
image.  Fig.~\ref{fig:boundary_show_in_paper} shows the resulting embeddings
in the first row of each panel.  Note that even though we didn't utilize
mean-shift grouping, the trained embedding already produces compact clusters.
To compare quantitatively to the state-of-the-art, we learn a fusion
layer that combines predictions from multiple levels of the feature
hierarchy fine-tuned with a logistic loss to match the binary output.
Fig.~\ref{fig:boundary_show_in_paper} shows the results in the second row.
Interestingly, we can see that the fine-tuned model embeddings
encode not only boundary presence/absence but also the orientation
and signed distance to nearby boundaries.

Quantitatively,
we compare our model to
COB~\cite{maninis2017convolutional},
HED~\cite{xie2015holistically},
CEDN~\cite{yang2016object},
LEP~\cite{najman1996geodesic},
UCM~\cite{arbelaez2011contour},
ISCRA~\cite{ren2013image},
NCuts~\cite{shi2000normalized},
EGB~\cite{felzenszwalb2004efficient},
and the original mean shift (MShift) segmentation algorithm~\cite{comaniciu2002mean}.
Fig.~\ref{fig:boundary_prcurve} shows standard benchmark precision-recall for
all the methods, demonstrating our model achieves state-of-the-art performance.
Note that our model has the same architecture of COB~\cite{maninis2017convolutional}
except with a different
loss functions and no explicit branches to compute boundary orientation.  Our
embedding loss by naturally pushes boundary pixel embeddings to be similar
which is also the desirable property for detecting boundaries using logistic
loss.  Note that it is possible to surpass human performance with several
sophisticated techniques~\cite{kokkinos2015pushing}, we don't pursue
this as it is out the scope of this paper.

\subsection{Object Proposal Detection}
\begin{figure}[t]
\centering
   \includegraphics[width=1\linewidth]{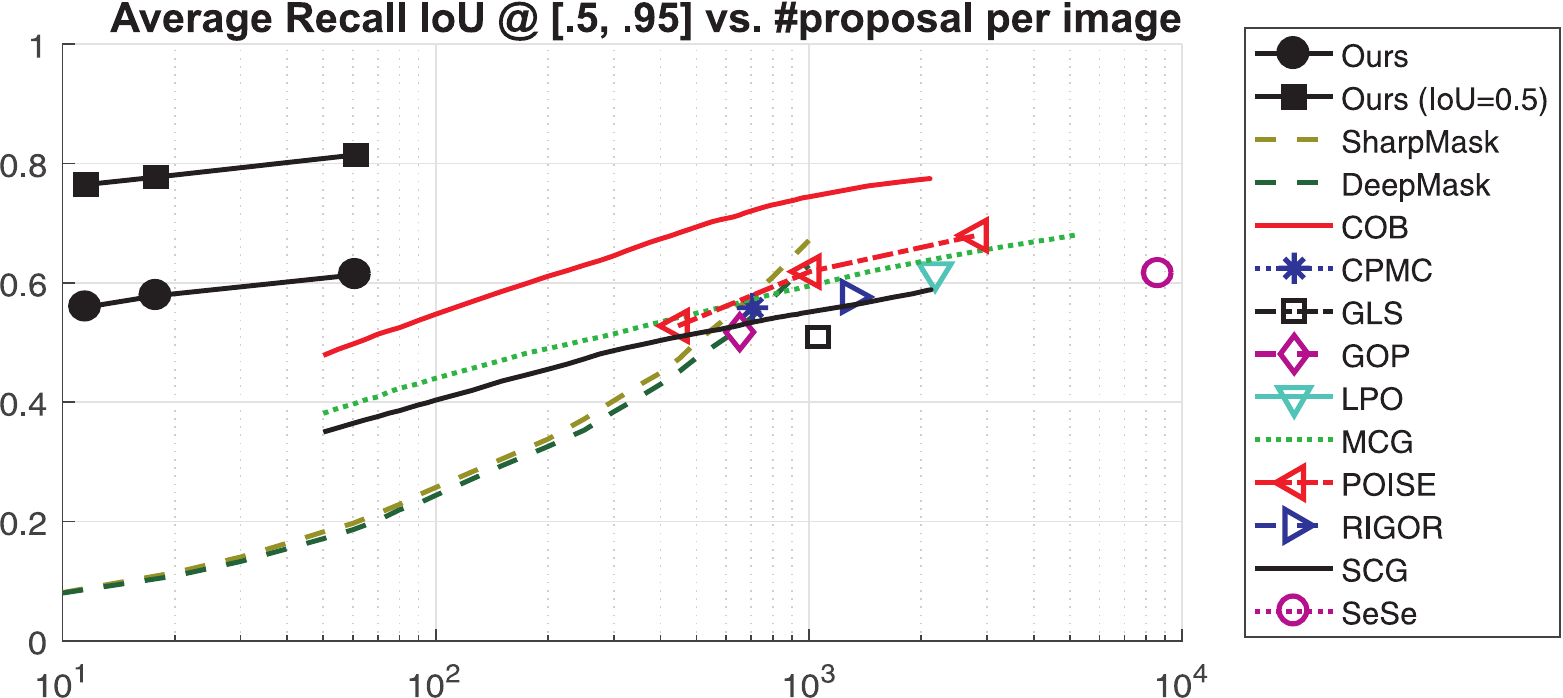}
   \vspace{-6mm}
   \caption{Segmented object proposals evaluation on PASCAL VOC 2012 validation
   set measured by Average Recall (AR) at IoU from 0.5 to 0.95 and step size as 0.5.
   We also include the curve for our method at IoU=0.5.
   }
\label{fig:plot_perf_objPropDet}
\end{figure}

Object proposals are an integral part of current object detection and semantic
segmentation pipelines~\cite{ren2015faster, he2017mask}, as they provide a
reduced search space of locations, scales and shapes for subsequent
recognition.  State-of-the-art methods usually involve training models that
output large numbers of proposals, particularly those based on bounding boxes.
Here we demonstrate that by training our framework with 64-dimensional
embedding space on the object instance level annotations, we are able to
produce very high quality object proposals by grouping the pixels into
instances.  It is worth noting that due to the nature of our grouping module,
far fewer number of proposals are produced with much higher quality.  We
compare against the most recent techniques including
POISE~\cite{humayun2015middle},
LPO~\cite{krahenbuhl2015learning},
CPMC~\cite{carreira2012cpmc},
GOP~\cite{krahenbuhl2014geodesic},
SeSe~\cite{uijlings2013selective},
GLS~\cite{rantalankila2014generating},
RIGOR~\cite{humayun2014rigor}.


Fig.~\ref{fig:plot_perf_objPropDet} shows the Average Recall
(AR)~\cite{hosang2016makes} with respect to the number of object
proposals\footnote{Our basic model produces $\sim10$ proposals per image.  In order
to plot a curve for our model for larger numbers of proposals, we run the mean
shift grouping with multiple smaller bandwidth parameters, pool the results,
and remove redundant proposals.}. Our model performs remarkably well compared
to other methods, achieving high average recall of ground-truth objects with two
orders of magnitude fewer proposals.  We also plot the curves for
SharpMask~\cite{pinheiro2015learning} and DeepMask~\cite{pinheiro2016learning}
using the proposals released by the authors. Despite only training on PASCAL,
we outperform these models which were trained on the much larger COCO dataset~\cite{lin2014microsoft}.
In Table~\ref{tab:objProposalDet} we report the total average recall at
IoU$=0.5$ for some recently proposed proposal detection methods, including
unpublished work inst-DML~\cite{fathi2017semantic} which is similar in spirit
to our model but learns a Euclidean distance based metric to group pixels.
We can clearly see that our method achieves significantly better results than
existing methods.

{
\setlength{\tabcolsep}{0.25em}
\begin{table}
\centering
{\footnotesize
\begin{tabular}{c | c c c  c |  c  }
\hline
 $\#$prop. & SCG~\cite{pont2017multiscale}
                     & MCG~\cite{pont2017multiscale}
                     & COB~\cite{maninis2017convolutional}
                     & inst-DML~\cite{fathi2017semantic}
                     & Ours      \\
\hline
10               & -            & -         & -     &  0.558     & 0.769 \\
60               & 0.624        & 0.652     & 0.738 &  0.667     & 0.814 \\
\hline
\end{tabular}
}
\vspace{-3mm}
\caption{Object proposal detection on PASCAL VOC 2012 validation set measured
by total Average Recall (AR) at IoU=0.50 and various number of proposals per image.}
\vspace{-2mm}
\label{tab:objProposalDet}
\end{table}
}

\begin{figure}[t]
\centering
   \includegraphics[width=0.80\linewidth]{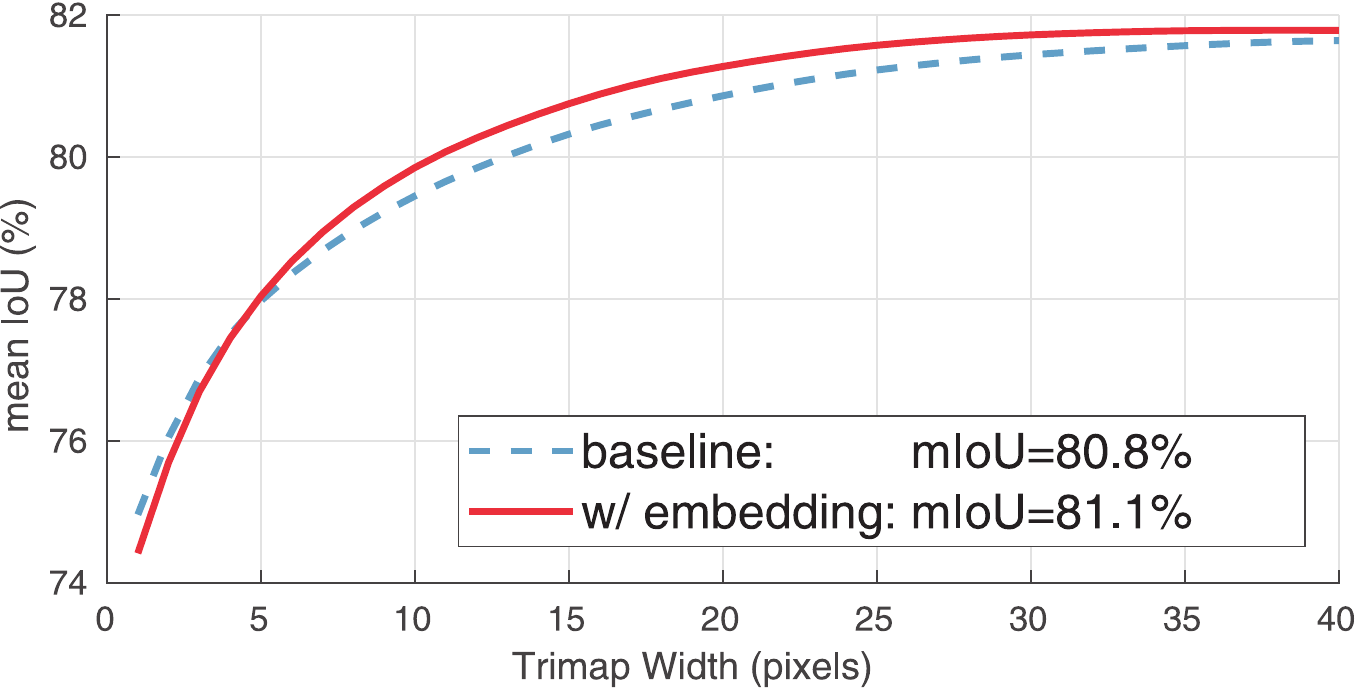} 
   \vspace{-2mm}
   \caption{Semantic segmentation performance as a function of distance from ground-truth
   object boundaries comparing a baseline model trained with cross-entropy loss versus a
   model which also includes embedding loss.
   }
\label{fig:trimap_semanticSeg}
\vspace{-1mm}
\end{figure}

\begin{figure}[t]
\centering
   \includegraphics[width=1\linewidth]{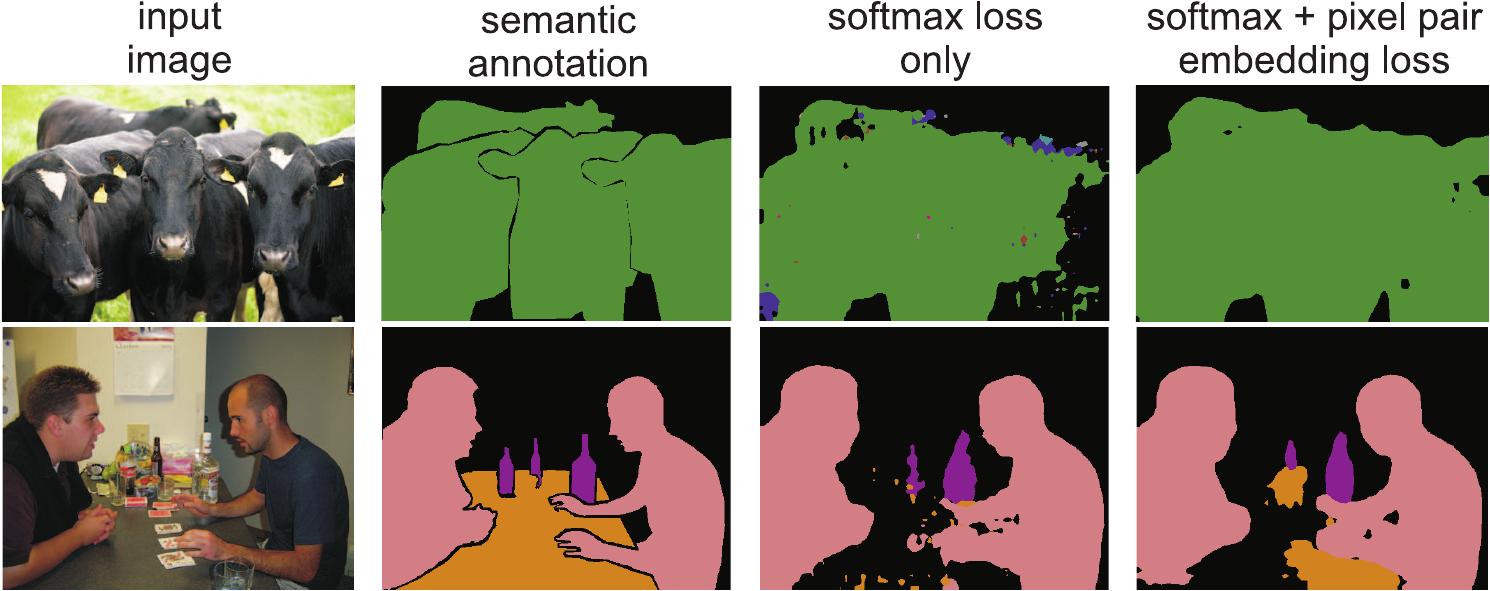} 
   \vspace{-5mm}
   \caption{The proposed embedding loss improves semantic segmentation by
   forcing the pixel feature vectors to be similar within the segments.
   Randomly selected images from PASCAL VOC2012 validation set.}
\label{fig:semantic_seg_demo}
\vspace{-3mm}
\end{figure}

\begin{figure*}[ht]
\centering
   \includegraphics[width=0.99\linewidth]{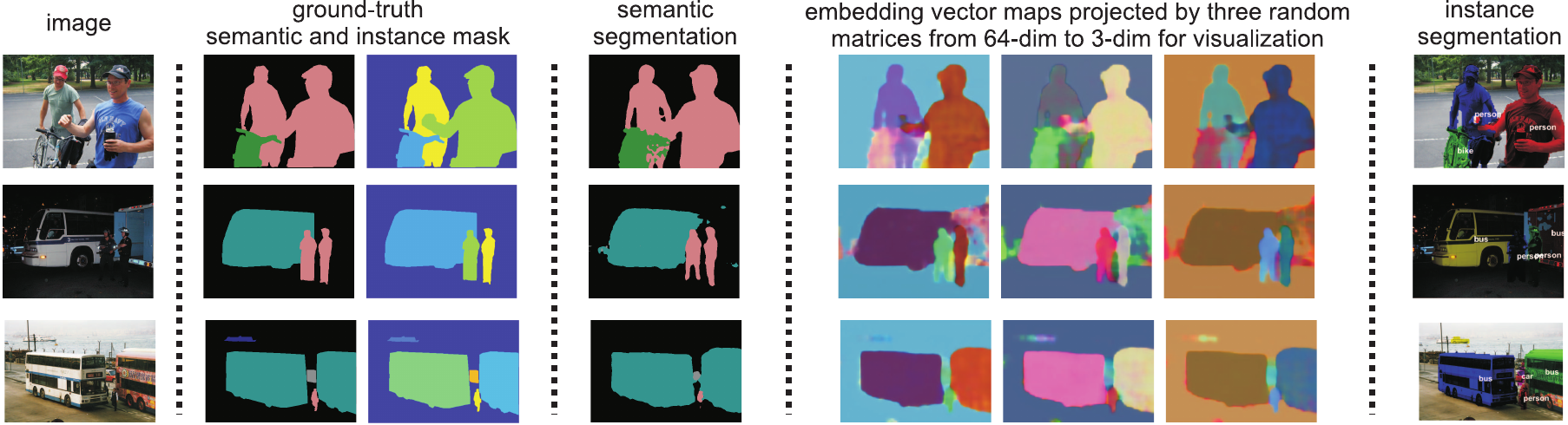}
   \vspace{-3mm}
   \caption{Visualization of generic/instance-level semantic segmentation on
   random PASCAL VOC 2012 validation images.
   }
\label{fig:instSeg_show_in_paper}
\vspace{-1mm}
\end{figure*}

{
\setlength{\tabcolsep}{0.3em} 
\begin{table*}[!htbp]
\footnotesize
\begin{center}
\begin{tabular}{| c | c c c c c c c c c c c c c c c c c c c c | c |}
\hline
\footnotesize{Method} & \begin{sideways}\footnotesize{plane}\end{sideways} & \begin{sideways}\footnotesize{bike}\end{sideways} & \begin{sideways}\footnotesize{bird}\end{sideways} & \begin{sideways}\footnotesize{boat}\end{sideways} & \begin{sideways}\footnotesize{bottle}\end{sideways} & \begin{sideways}\footnotesize{bus}\end{sideways} & \begin{sideways}\footnotesize{car}\end{sideways} & \begin{sideways}\footnotesize{cat}\end{sideways} & \begin{sideways}\footnotesize{chair}\end{sideways} & \begin{sideways}\footnotesize{cow}\end{sideways} & \begin{sideways}\footnotesize{table}\end{sideways} & \begin{sideways}\footnotesize{dog}\end{sideways} & \begin{sideways}\footnotesize{horse}\end{sideways} & \begin{sideways}\footnotesize{motor}\end{sideways} & \begin{sideways}\footnotesize{person}\end{sideways} & \begin{sideways}\footnotesize{plant}\end{sideways} & \begin{sideways}\footnotesize{sheep}\end{sideways} & \begin{sideways}\footnotesize{sofa}\end{sideways} & \begin{sideways}\footnotesize{train}\end{sideways} & \begin{sideways}\footnotesize{tv}\end{sideways}  & \begin{sideways}\footnotesize{mean}\end{sideways} \\ \hline \hline
     SDS \cite{hariharan2014simultaneous}         & 58.8 & 0.5  & 60.1 & 34.4 & 29.5 & 60.6 & 40.0 & 73.6 & 6.5  & 52.4 & 31.7 & 62.0 & 49.1 & 45.6 & 47.9 & 22.6 & 43.5 & 26.9 & 66.2 & 66.1 & 43.8 \\
     Chen et al. \cite{chen2015multi} & 63.6 & 0.3  & 61.5 & 43.9 & 33.8 & 67.3 & 46.9 & 74.4 & 8.6  & 52.3 & 31.3 & 63.5 & 48.8 & 47.9 & 48.3 & 26.3 & 40.1 & 33.5 & 66.7 & 67.8 & 46.3 \\
     PFN \cite{liang2015proposal}     & 76.4 & 15.6 & 74.2 & 54.1 & 26.3 & 73.8 & 31.4 & 92.1 & 17.4 & 73.7 & 48.1 & 82.2 & 81.7 & 72.0 & 48.4 & 23.7 & 57.7 & 64.4 & 88.9 & 72.3 & 58.7 \\
     MNC \cite{dai2016instance}            & -    & -    & -    & -    & -    & -    & -    & -    & -    & -    & -    & -    & -    & -    & -    & -    & -    & -    & -    & -    & 63.5 \\
     Li et al. \cite{li2016fully}     & -    & -    & -    & -    & -    & -    & -    & -    & -    & -    & -    & -    & -    & -    & -    & -    & -    & -    & -    & -    & 65.7 \\
     R2-IOS \cite{liang2016reversible}          & 87.0 & 6.1  & 90.3 & 67.9 & 48.4 & 86.2 & 68.3 & 90.3 & 24.5 & 84.2 & 29.6 & 91.0 & 71.2 & 79.9 & 60.4 & 42.4 & 67.4 & 61.7 & 94.3 & 82.1 & 66.7 \\
     Assoc. Embed.~\cite{newell2016associative} & -    & -    & -    & -    & -    & -    & -    & -    & -    & -    & -    & -    & -    & -    & -    & -    & -    & -    & -    & -    & 35.1 \\
     inst-DML~\cite{fathi2017semantic}                             & 69.7    & 1.2    & 78.2    & 53.8    & 42.2    & 80.1    & 57.4    & 88.8    & 16.0    & 73.2    & 57.9    & 88.4    & 78.9    & 80.0    & 68.0    & 28.0    & 61.5    & 61.3    & 87.5    & 70.4    & 62.1    \\
     \hline
     Ours
        & 85.9 & 10.0 & 74.3 & 54.6 & 43.7
        & 81.3 & 64.1 & 86.1 & 17.5 & 77.5
        & 57.0 & 89.2 & 77.8 & 83.7 & 67.9
        & 31.2 & 62.5 & 63.3 & 88.6 & 74.2
     & 64.5   \\
    \hline

\end{tabular}
\vspace{-5mm}
\end{center}
  \caption{Instance-level segmentation comparison using APr metric at 0.5 IoU on the PASCAL VOC 2012 validation set.}
\vspace{-2mm}
\label{tab:results_iou}
\end{table*}
}

\subsection{Semantic Instance Detection}

As a final test of our method, we also train it to produce semantic labels
which are combined with our instance proposal method to recognize the detected
proposals.

For semantic segmentation which is a k-way classification problem, we train a
model using cross-entropy loss alongside our embedding loss.  Similar to our
proposal detection model, we use a 64-dimension embedding space on top of
DeepLab-v3~\cite{chen2017rethinking} as our base model.  While there are more
complex methods in literature such as PSPNet~\cite{zhao2016pyramid} and which
augment training with additional data (e.g., COCO~\cite{lin2014microsoft} or
JFT-300M dataset~\cite{sun2017revisiting}) and utilize ensembles and
post-processing, we focus on a simple experiment training the base model
with/without the proposed pixel pair embedding loss to demonstrate the
effectiveness. 

In addition to reporting mean intersection over union (mIoU) over all
classes,
we also computed mIoU restricted to a narrow band of pixels around the
ground-truth boundaries.  This partition into figure/boundary/background is
sometimes referred to as a tri-map in the matting literature and has been
previously utilized in analyzing semantic segmentation
performance~\cite{kohli2009robust, chen2016deeplab, ghiasi2016laplacian}.
Fig.~\ref{fig:trimap_semanticSeg} shows the mIoU as a function of the width
of the tri-map boundary zone.  This demonstrates that with embedding loss
yields performance gains over cross-entropy primarily far from ground-truth
boundaries where it successfully fills in holes in the segments output
(see also qualitative results in Fig.~\ref{fig:semantic_seg_demo}).
This is in spirit similar to the model in \cite{harley2017segmentation},
which considers local consistency to improve spatial precision.
However, our uniform sampling allows for long-range interactions between pixels.

To label detected instances with semantic labels, we use the semantic
segmentation model described above to generate labels and then use a simple
voting strategy to transfer these predictions to the instance proposals.  In
order to produce a final confidence score associated with each proposed object,
we train a linear regressor to score each object instance based on its
morphology (e.g., size, connectedness) and the consistency w.r.t. the semantic
segmentation prediction. We note this is substantially simpler than approaches
based, e.g. on Faster-RCNN~\cite{ren2015faster} which use much richer convolutional features to
rescore segmented instances~\cite{he2017mask}.

Comparison of instance detection performance are displayed in
Table~\ref{tab:results_iou}.  We use a standard IoU threshold of 0.5 to
identify true positives, unless an ground-truth instance has already been
detected by a higher scoring proposal in which case it is a false positive.  We
report the average precision per-class as well as the average all classes (as
in~\cite{hariharan2011semantic}).  Our approach yields competitive performance
on VOC validation despite our simple re-scoring.  Among the competing methods,
the one closest to our model is inst-DML~\cite{fathi2017semantic}, that learns
Euclidean distance based metric with logistic loss.  The inst-DML approach
relies on generating pixel seeds to derive instance masks.  The pixel seeds may
fail to correctly detect thin structures which perhaps explains why this method
performs 10x worse than our method on the bike category. In contrast, our
mean-shift grouping approach doesn't make strong assumptions about the
object shape or topology.

For visualization purposes, we generate three random matrices projections of
the 64-dimensional embedding and display them in the spatial domain as RGB
images.  Fig.~\ref{fig:instSeg_show_in_paper} shows the embedding
visualization, as well as predicted semantic segmentation and instance-level
segmentation.  From the visualization, we can see the instance-level semantic
segmentation outputs complete object instances even though semantic
segmentation results are noisy, such as the bike in the first image in
Fig.~\ref{fig:instSeg_show_in_paper}. The instance embedding provides
important details that resolve both inter- and intra-class instance overlap
which are not emphasized in the semantic segmentation loss.

%% file: section06_conclusion.tex
\section{Conclusion and Future Work}
We have presented an end-to-end trainable framework for solving pixel-labeling
vision problems based on two novel contributions: a pixel-pairwise loss based
on spherical max-margin embedding and a variant of mean shift grouping embedded
in a recurrent architecture. These two components mesh closely to provide a
framework for robustly recognizing variable numbers of instances without
requiring heuristic post-processing or hyperparameter tuning to account for
widely varying instance size or class-imbalance. The approach is simple and
amenable to theoretical analysis, and when coupled with standard architectures
yields instance proposal generation which substantially outperforms
state-of-the-art. Our experiments demonstrate the potential for instance
embedding and open many opportunities for future work including learn-able
variants of mean-shift grouping, extension to other pixel-level domains such as
encoding surface shape, depth and figure-ground and multi-task embeddings.




%% file: section07_acknowledgement.tex
\section*{Acknowledgement}
This project is supported by NSF grants
IIS-1618806, IIS-1253538, DBI-1262547 and a hardware donation
from NVIDIA.
Shu Kong personally thanks Mr. Kevis-Kokitsi Maninis,
Dr. Alireza Fathi, Dr. Kevin Murphy and Dr. Rahul Sukthankar for the helpful discussion, advice and encouragement.

%% file: section07_appendix4arxiv.tex
\setcounter{section}{0}
\setcounter{propositions}{0}

\section{Analysis of Pairwise Loss for Spherical Embedding}
In this section, we provide proofs for the propositions presented in the paper
which provide some analytical understanding of our proposed objective function,
and the mechanism for subsequent pixel grouping mechanism.

\begin{propositions}
\label{theorem:ObjLowerBound_appendix}
For $n$ vectors $\{\x_1, \dots, \x_n\}$, the total intra-pixel similarity is
bounded as $\sum_{i\neq j} \x_i^T\x_j \ge -\sum_{i=1}^n  \Vert \x_i \Vert_2^2$.
In particular, for $n$ vectors on the hypersphere where $\Vert\x_i\Vert_2 = 1$,
we have $\sum_{i\neq j} \x_i^T\x_j \ge -n$.
\end{propositions}

\begin{proof}
  First note that $\Vert \x_1 + \dots + \x_n \Vert_2^2 \ge 0$.  We expand
  the square and collect all the cross terms so we have $\sum_{i} \x_i^T\x_i +
  \sum_{i\not=j}\x_i^T\x_j \ge 0$.  Therefore, $\sum_{i\not=j}\x_i^T\x_j \ge
  -\sum_{i=1}^n \Vert \x_i\Vert_2^2$.  When all the vectors are on the
  hyper-sphere, i.e. $\Vert \x_i\Vert_2 = 1$, then $\sum_{i\not=j}\x_i^T\x_j
  \ge -\sum_{i=1}^n \Vert \x_i\Vert_2^2= -n$. $\hfill\blacksquare$
\end{proof}



\begin{propositions}
\label{lemma:max_margin_appendix}
If $n$ vectors $\{\x_1, \dots, \x_n\}$ are distributed on a 2-sphere (i.e.
$\x_i \in \RB^3$ with $\Vert\x_i\Vert_2 = 1, \forall i=1\dots n$) then the
similarity between any pair is lower-bounded by $s_{ij} \geq
1-\Big(\frac{2\pi}{\sqrt{3}n} \Big)$. Therefore, choosing the parameter
$\alpha$ in the maximum margin term in objective function to be less
than
$1-\Big(\frac{2\pi}{\sqrt{3}n} \Big)$ results in positive loss even for
a perfect embedding of $n$ instances.
\end{propositions}

We treat all the $n$ vectors as representatives of $n$ different instances in
the image and seek to minimize pairwise similarity, or equivalently maximize
pairwise distance (referred to as Tammes's problem, or the hard-spheres
problem~\cite{saff1997distributing}).

\begin{proof}
Let $d  = \max\limits_{\{\x_i\}} \min\limits_{i\not=j}\Vert \x_i-\x_j\Vert_2$
be the distance between the closest point pair of the optimally distributed
points.  Asymptotic results in~\cite{habicht1951lagerung} show that, for some
constant $C>0$,
\begin{equation}
  \Big(\frac{8\pi}{\sqrt{3}n}\Big)^{\frac{1}{2}} -Cn^{-\frac{2}{3}} \le d \le \Big( \frac{8\pi}{\sqrt{3}n} \Big)^{\frac{1}{2}}
\end{equation}
Since $\Vert \x_i-\x_j\Vert_2^2=2-2\x_i^T \x_j$, we can rewrite this bound in
terms of the similarity
$s_{ij} = \frac{1}{2}\left(1 + \frac{\x_i^T\x_j}{\Vert\x_i\Vert_2 \Vert\x_j\Vert_2}\right)$,
so that for any $i \not= j$:
\begin{equation}
1- \Big( \frac{2\pi}{\sqrt{3}N} \Big) \le s_{ij} \le  1- \frac{1}{4} \Bigg( \Big(\frac{8\pi}{\sqrt{3}N}\Big)^{\frac{1}{2}} -CN^{-\frac{2}{3}}\Bigg)^2
\end{equation}
Therefore, choosing $\alpha \leq 1- \Big( \frac{2\pi}{\sqrt{3}N} \Big)$,
guarantees that $[s_{ij}-\alpha]_{+}\geq 0$ for some pair $i \not= j$. Choosing
$\alpha > 1- \frac{1}{4} \Bigg( \Big(\frac{8\pi}{\sqrt{3}N}\Big)^{\frac{1}{2}}
-CN^{-\frac{2}{3}}\Bigg)^2$, guarantees the existence of an embedding with
$[s_{ij}-\alpha]_{+}=0$.
$\hfill\blacksquare$
\end{proof}

\begin{figure*}[t]
\centering
   \includegraphics[width=1\linewidth]{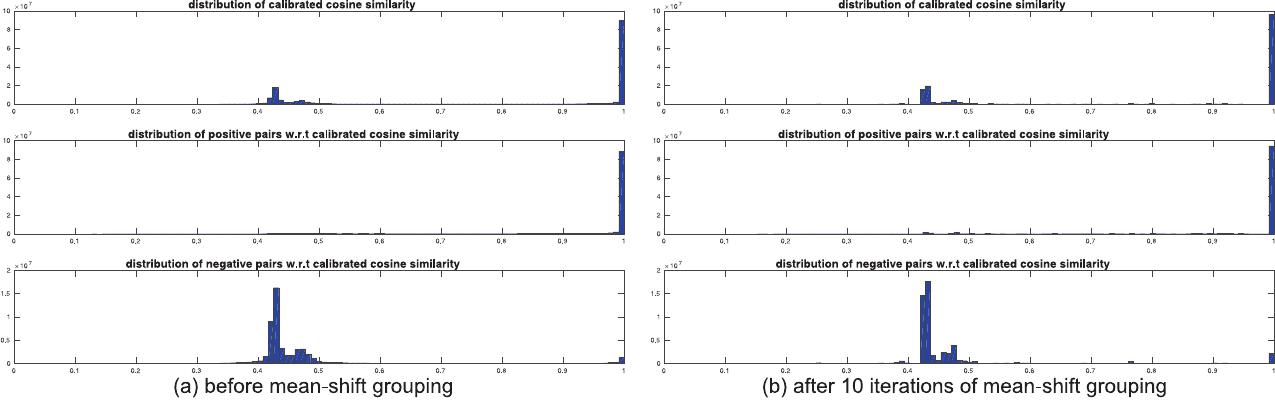}
   \caption{Distribution of calibrated cosine similarity between pairs of pixels.
   After 10 iterations of mean-shift grouping.  Margin is 0.5 for negative
   pairs.  From the figures, we believe that the mean shift grouping mechanism
   forces learning to focus on those pixel pairs that will not be corrected by
   mean shift grouping itself if running offline, and thus pushing down to
   parameters in the the deep neural network to learn how to correct them
   during training.
   }
\label{fig:pixel_pair_distr}
\end{figure*}

\section{Details of Recurrent Mean Shift Grouping }

There are two commonly used multivariate kernels in mean shift algorithm.  The
first, Epanechnikov kernel~\cite{epanechnikov1969non,cheng1995mean}, has the
following profile
\begin{equation}
K_E(\x)=
\begin{cases}
    \frac{1}{2} c_d^{-1} (d+2) (1-\Vert\x \Vert_2^2),& \text{if } \Vert\x\Vert_2\le 1\\
    0,              & \text{otherwise}
\end{cases}
\end{equation}
where $c_d$ is the volume of the unit $d$-dimensional sphere. The standard
mean-shift algorithm computes the gradient of the kernel density estimate
given by
\[
p(\x) = \frac{1}{N} \sum_{i=1}^N K_E(\frac{\x - \x_i}{b})
\]
and identifies modes (local maxima) where $\nabla p(\x)= 0$. The scale
parameter $b$ is known as the kernel bandwidth and determines the
smoothness of the estimator.  The gradient of $p(\x)$ can be elegantly computed
as the difference between $\x$ and the mean of all data points with $\Vert \x -
\x_i \Vert \leq b$, hence the name ``mean-shift'' for performing gradient
ascent.

Since the Epanechnikov profile is not differentiable at the boundary, we
use the squared exponential kernel adapted to vectors on the sphere:
\begin{equation}
\begin{split}
K(\x,\x_i) & \propto \exp( \delta^2 \x^T \x_i ) \\
\end{split}
\end{equation}
which can be viewed as a natural extension of the Gaussian to spherical data
(known as the von Mises Fisher (vMF)
distribution~\cite{fisher1953dispersion,banerjee2005clustering,mardia2009directional,kobayashi2010mises}).
In our experiments we set the bandwidth $\delta$ based on the margin $\alpha$
so that $\frac{1}{\delta}=\frac{1-\alpha}{3}$.

Our proposed algorithm also differs from the standard mean-shift clustering
(i.e., ~\cite{comaniciu1999mean}) in that rather than performing gradient
ascent on a fixed kernel density estimate $p(\x)$, at every iteration we
alternate between updating the embedding vectors $\{\x_i\}$ using gradient
ascent on $p(\x)$ and re-estimating the density $p(\x)$ for the updated
vectors.  This approach is termed Gaussian Blurring Mean Shift (GBMS) in
~\cite{carreira2008generalised} and has converge rate guarantees for
data which starts in compact clusters.

In the paper we visualized embedding vectors after GBMS for specific
examples.  Figure~\ref{fig:pixel_pair_distr} shows aggregate statistics over a
collection of images (in the experiment of instance segmentation).
We plot the distribution of pairwise similarities for
positive and negative pairs during forward propagation through 10 iterations.
We can observe that the mean shift module produces sharper distributions,
driving the similarity between positive pairs to 1 making it trivial to
identify instances.

\subsection{Gradient Calculation for Recurrent Mean Shift}

To backpropagate gradients through an iteration of GBMS, we break the
calculation into a sequence of steps below where we assume the vectors
in the data matrix $X$ have already been normalized to unit length.
\begin{equation}\small
\begin{split}
&\colorboxed{red}{
\begin{split}
\S = & \X^T\X \\
\end{split}
}\\
&\colorboxed{blue}{
\begin{split}
\K = & \exp(\delta^2 \S) \\
\end{split}
}, \\
&\colorboxed{green}{
\begin{split}
{\bf d} = & \K^T\1 \\
\end{split}
}\\
& \colorboxed{black}{
\begin{split}
\q = & {\bf d}^{-1} \\
\end{split}
}\\
&\colorboxed{orange}{
\begin{split}
{\bf P} = (1-\eta)\I + \eta\K \diag(\q)\\
\end{split}
}\\
&\colorboxed{cyan}{
\begin{split}
\Y = & \X{\bf P} \\
\end{split}
}\\
\end{split}
\label{eq:GBMS_matrix_form}
\end{equation}
where $\Y$  is the updated data after one iteration which is subsequently
renormalized to project back onto the sphere.  Let $\ell$ denote the loss
and $\odot$ denote element-wise product. Backpropagation gradients are
then given by:
\begin{equation}\small
\begin{split}
&\colorboxed{red}{
\begin{split}
&\frac{\partial \ell}{\partial \X} =  2\X \frac{\partial \ell}{\partial \S} \\
\end{split}
}\\
&\colorboxed{blue}{
\begin{split}
&\frac{\partial \ell}{\partial\S} =  \delta^2\exp(\delta^2\S) \odot \frac{\partial \ell}{\partial\K} \\
&\frac{\partial \ell}{\partial\delta} =  2\delta \sum_{ij} \Big( (s_{ij}) \odot \exp(\delta^2 s_{ij}) \odot \frac{\partial \ell}{\partial {k_{ij}}} \Big)\\
\end{split}
}\\
&\colorboxed{green}{
\begin{split}
&\frac{\partial \ell}{\partial{\bf K}} =  \1 \Big( \frac{\partial \ell}{\partial {\bf d}} \Big)^T \\
\end{split}
}\\
&\colorboxed{black}{
\begin{split}
&\frac{\partial \ell}{\partial{\bf d}} =  \frac{\partial \ell}{\partial{\bf q}} \odot (-{\bf d}^{-2}) \\
\end{split}
}\\
&\colorboxed{orange}{
\begin{split}
&\frac{\partial \ell}{\partial{\bf K}} =  \eta\Big(\frac{\partial \ell}{\partial{\bf P}}\Big) (\q\1^T)\\
&\frac{\partial \ell}{\partial{\bf q}} =  \eta\Big(\frac{\partial \ell}{\partial{\bf P}}\Big)^T \K\1 \\
\end{split}
}\\
&\colorboxed{cyan}{
\begin{split}
&\frac{\partial \ell}{\partial{\bf X}} =  \frac{\partial \ell}{\partial{\bf Y}} {\bf P}^T \\
&\frac{\partial \ell}{\partial{\bf P}} =  \X^T \frac{\partial \ell}{\partial{\bf Y}} \\
\end{split}
}
\end{split}
\end{equation}

%
%
%

\subsection{Toy Example of Mean Shift Backpropagation}
In the paper we show examples of the gradient vectors backpropagated through
recurrent mean shift to the initial embedding space. Backpropagation through
this fixed model modulates the loss on the learned embedding, increasing the
gradient for initial embedding vectors whose instance membership is ambiguous
and decreasing the gradient for embedding vectors that will be correctly
resolved by the recurrent grouping phase.

Figure \ref{fig:simulation_trajectory} shows a toy example highlighting the difference
between supervised and unsupervised clustering.  We generate  a set of 1-D data
points drawn from three Gaussian distributions with mean and standard deviation
as $(\mu=3, \sigma=0.2)$, $(\mu=4, \sigma=0.3)$ and $(\mu=5, \sigma=0.1)$,
respectively, as shown in Figure \ref{fig:simulation_trajectory} (a).
We use mean squared error for the loss with a fixed linear
regressor $y_i = 0.5*x_i-0.5$ and fixed target labels. The optimal embedding
would set $x_i=3$ if $y_i=1$, and $x_i=5$ if $y_i=2$. We perform 30 gradient
updates of the embedding vectors $x_i \leftarrow x_i - \alpha \nabla_{x_i}
\ell$ with a step size $\alpha$ as 0.1.
We analyze the behavior of Gaussian Blurring Mean Shift (GBMS) with bandwidth as $0.2$.

If running GBMS for unsupervised clustering on these data with the default setting (bandwidth is 0.2),
we can see they are grouped into three piles,
as shown in Figure \ref{fig:simulation_trajectory} (b).
If updating the data using gradient descent without GBMS inserted,
we end up with three visible clusters even though the data move towards the ideal embedding in terms of classification.
Figure \ref{fig:simulation_trajectory} (c) and (d) depict the trajectories of 100 random data points during the 30 updates and the final result, respectively.

Now we insert the GBMS module to update these data with different loops,
and compare how this effects the performance.
We show the updated data distributions and those after five loops of GBMS grouping in column (e) and (f) of Figure~\ref{fig:simulation_trajectory},
respectively.
We notice that, with GBMS,
all the data are grouped into two clusters;
while with GBMS grouping they become more compact and are located exactly on
the ``ideal spot'' for mapping into label space (i.e. 3 and 5) and achieving
zero loss.
On the other hand,
we also observe that, even though these settings incorporates different number of GBMS loops,
they achieve similar visual results in terms of clustering the data.
To dive into the subtle difference,
we randomly select 100 data and depict their trajectories in column (g) and (h)
of Figure~\ref{fig:simulation_trajectory},
using a single loss on top of the last GBMS loop or multiple losses over every GBMS loops,
respectively.
We have the following observations:
\begin{enumerate}
  \item By comparing with Figure~\ref{fig:simulation_trajectory} (c),
    which depicts update trajectories without GBMS,
    GBMS module provides larger gradient to update those data further from their ``ideal spot'' under both scenarios.
  \item From (g),
    we can see the final data are not updated into tight groups.
    This is because that the updating mechanism only sees data after (some loops of) GBMS,
    and knows that these data will be clustered into tight groups through GBMS.
  \item A single loss with more loops of GBMS provides greater gradient than that with fewer loops to update data, as seen in (g).
  \item With more losses over every loops of GBMS,
    the gradients become even larger that the data are grouped more tightly and more quickly.
    This is because that the updating mechanism also incorporates the gradients from the loss over the original data,
    along with those through these loops of GBMS.
\end{enumerate}

To summarize,
our GBMS based recurrent grouping module indeed provides meaningful gradient during training with back-propagation.
With the convergent dynamics of GBMS,
our grouping module becomes especially more powerful in learning to group data with suitable supervision.

\begin{figure*}[t]
\centering
   \includegraphics[width=0.9\linewidth]{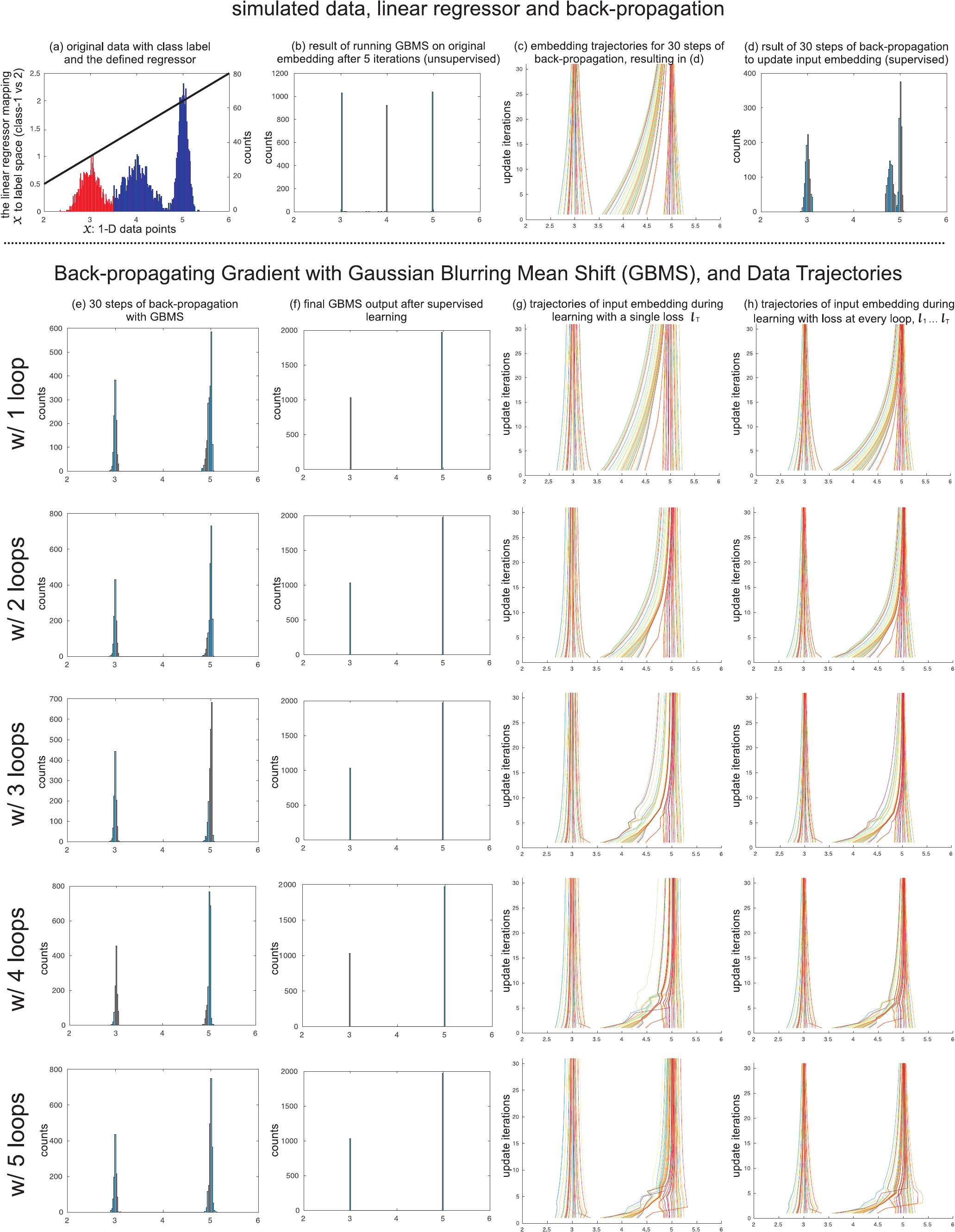}
   \caption{
   Trajectory of updating data using back-propagation without mean shift module (top row),
   and with the Gaussian Blurring Mean Shift (GBMS).
   To compare the results,
   we vary the number of GBMS loops in the grouping module,
   and use either a single loss at the final GBMS loop or multiple losses on all GBMS loops.
   All the configurations can shift data towards the ``ideal spots'' (3 or 5 depending on the label) in terms of the fixed regressor.
   }
\label{fig:simulation_trajectory}
\end{figure*}

\section{Additional Boundary Detection Results}

We show additional boundary detection results\footnote{Paper with high-resolution figures can be found at the
\href{http://www.ics.uci.edu/~skong2/SMMMSG.html}{Project Page}.} on BSDS500
dataset~\cite{arbelaez2011contour} based on our model in Figure
\ref{fig:boundary_show_appendix_part1},
\ref{fig:boundary_show_appendix_part3},
\ref{fig:boundary_show_appendix_part4}
and \ref{fig:boundary_show_appendix_part5}.
Specifically, besides showing the boundary detection result, we also show
3-dimensional pixel embeddings as RGB images before and after fine-tuning using
logistic loss.  From the consistent colors, we can see (1) our model
essentially carries out binary classification even using the pixel pair
embedding loss; (2) after fine-tuning with logistic loss, our model captures
also boundary orientation and signed distance to the boundary.  Figure
\ref{fig:demo_test_boundary} highlights this observation for an example image
containing round objects.  By zooming in one plate, we can observe a ``colorful
Mobius ring'', indicating the embedding features for the boundary also capture
boundary orientation and the signed distance to the boundary.

\begin{figure*}[t]
\centering
   \includegraphics[width=0.7\linewidth]{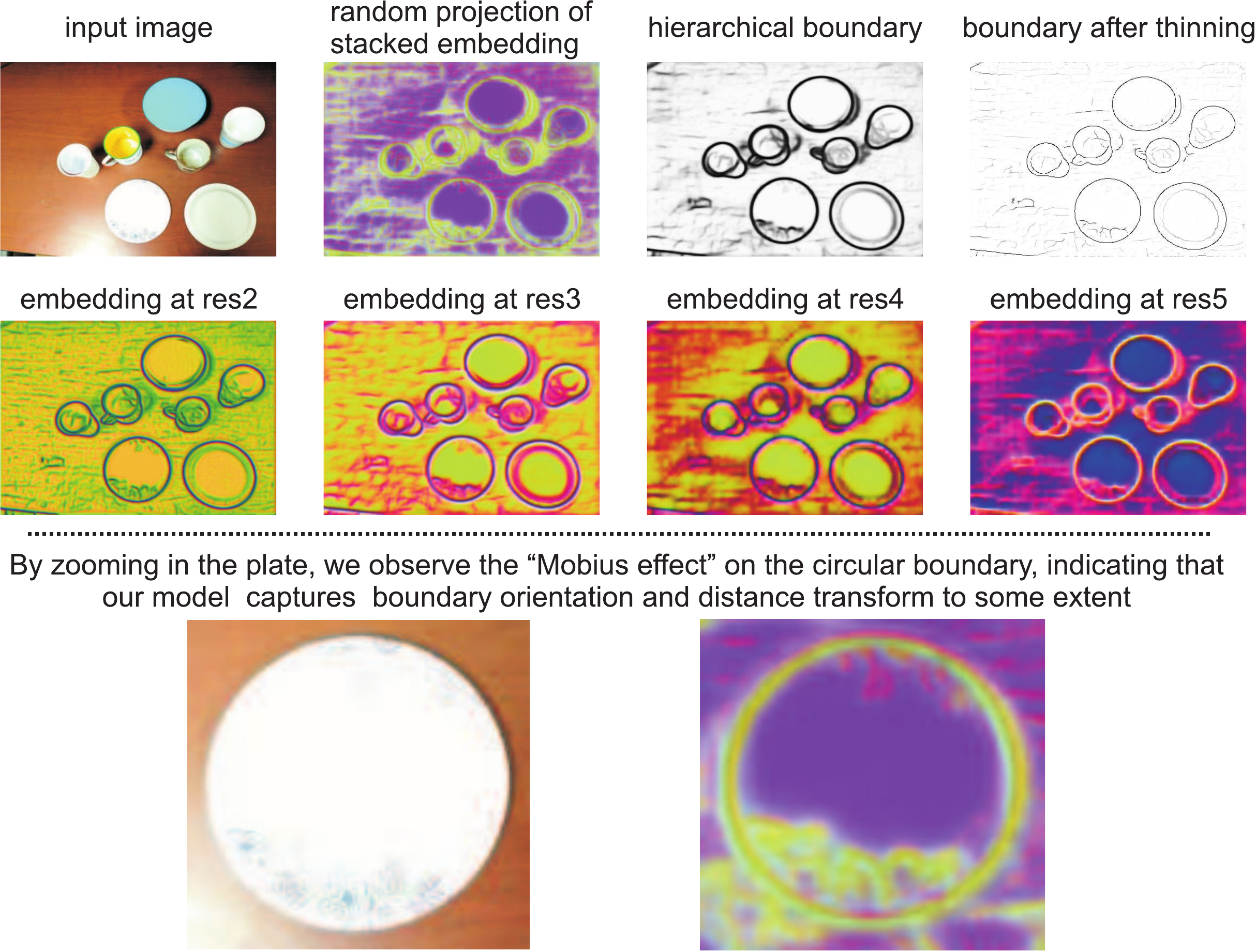}
   \vspace{-3mm}
   \caption{An image highlighting the structure of the embedding for an image
   with circular boundaries. We observe a ``Mobius effect'' where the embedding
   encodes both the orientation and distance to the boundary.}
\label{fig:demo_test_boundary}
\end{figure*}

\begin{figure*}[t]
\centering
   \includegraphics[width=0.7\linewidth]{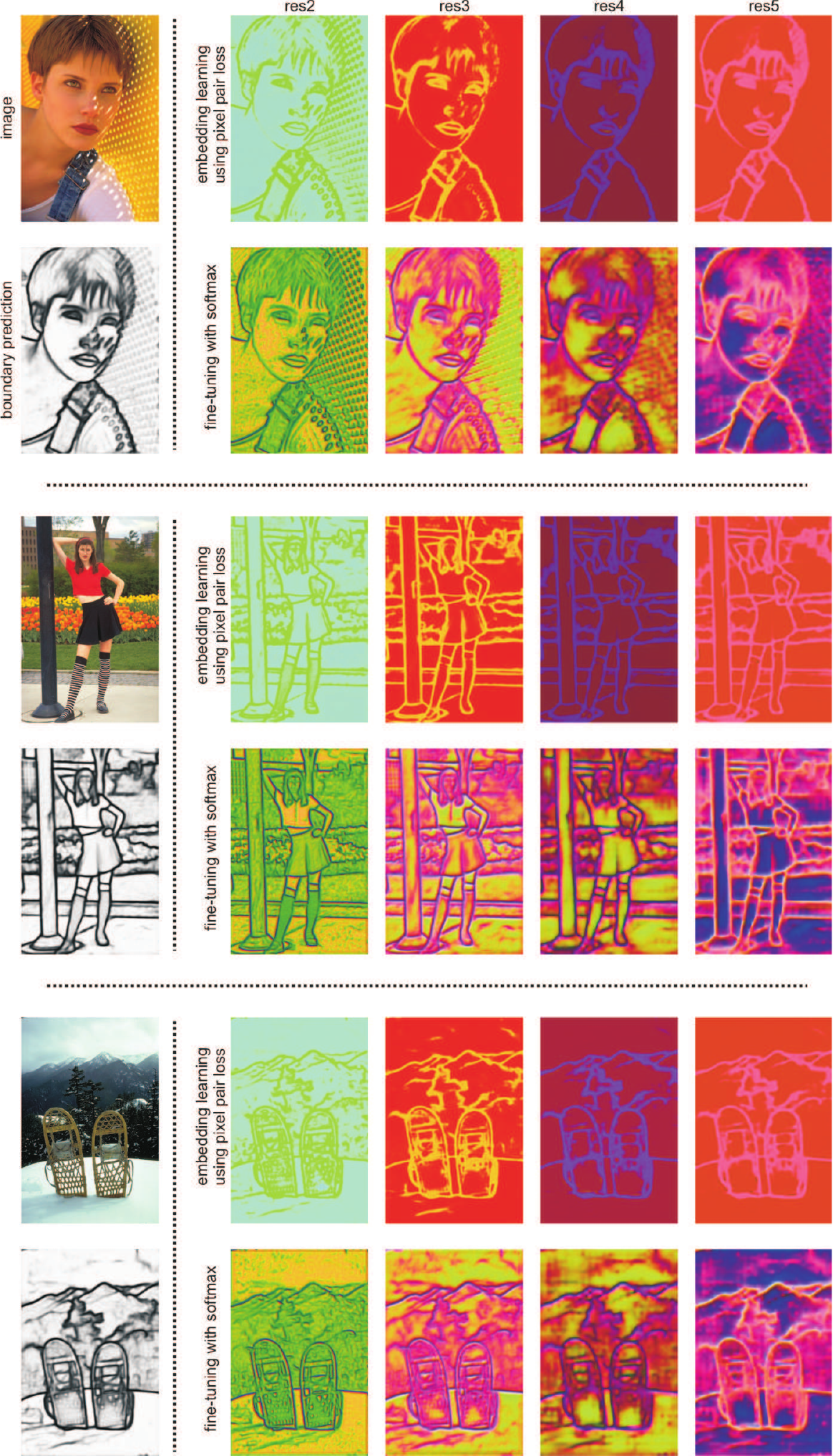}
   \vspace{-3mm}
   \caption{Visualization for boundary detection (part-$1/5$).
   Images are randomly selected from BSDS500 test set.  For each image, we show
   the embedding vectors  at different layers from the model before and after
   fine-tuning using logistic loss.  We can see that the boundary embedding
   vectors after fine-tuning not only highlights the boundary pixels, but also
   captures to some extent the edge orientation and distance from the colors
   conveyed.
   }
\label{fig:boundary_show_appendix_part1}
\end{figure*}


\begin{figure*}[t]
\centering
   \includegraphics[width=0.7\linewidth]{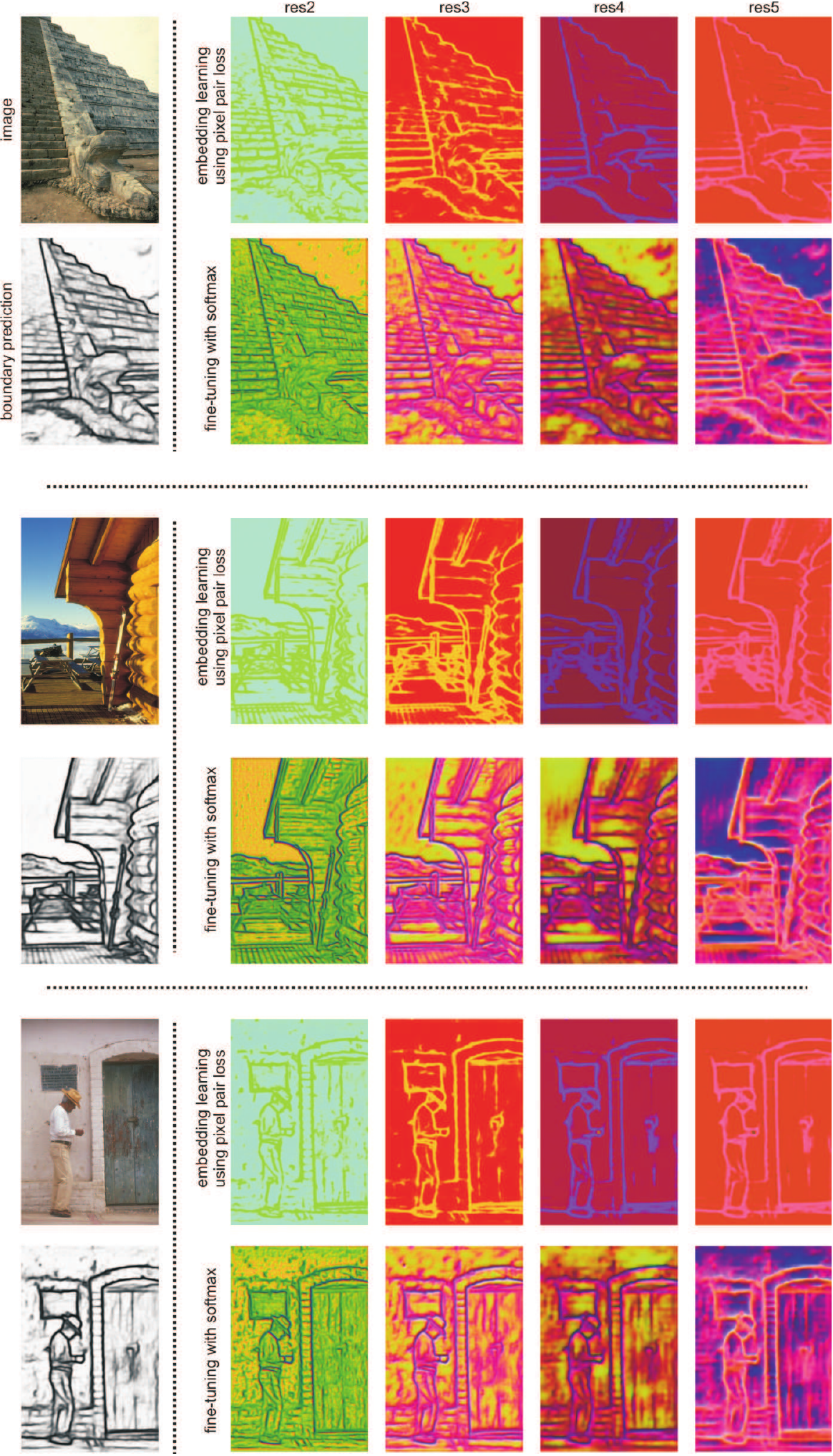}
   \vspace{-3mm}
   \caption{Visualization for boundary detection ($3/5$).
   Images are randomly selected from BSDS500 test set.  For each image, we show
   the embedding vectors  at different layers from the model before and after
   fine-tuning using logistic loss.  We can see that the boundary embedding
   vectors after fine-tuning not only highlights the boundary pixels, but also
   captures to some extent the edge orientation and distance from the colors
   conveyed.
   }
\label{fig:boundary_show_appendix_part3}
\end{figure*}

\begin{figure*}[t]
\centering
   \includegraphics[width=0.7\linewidth]{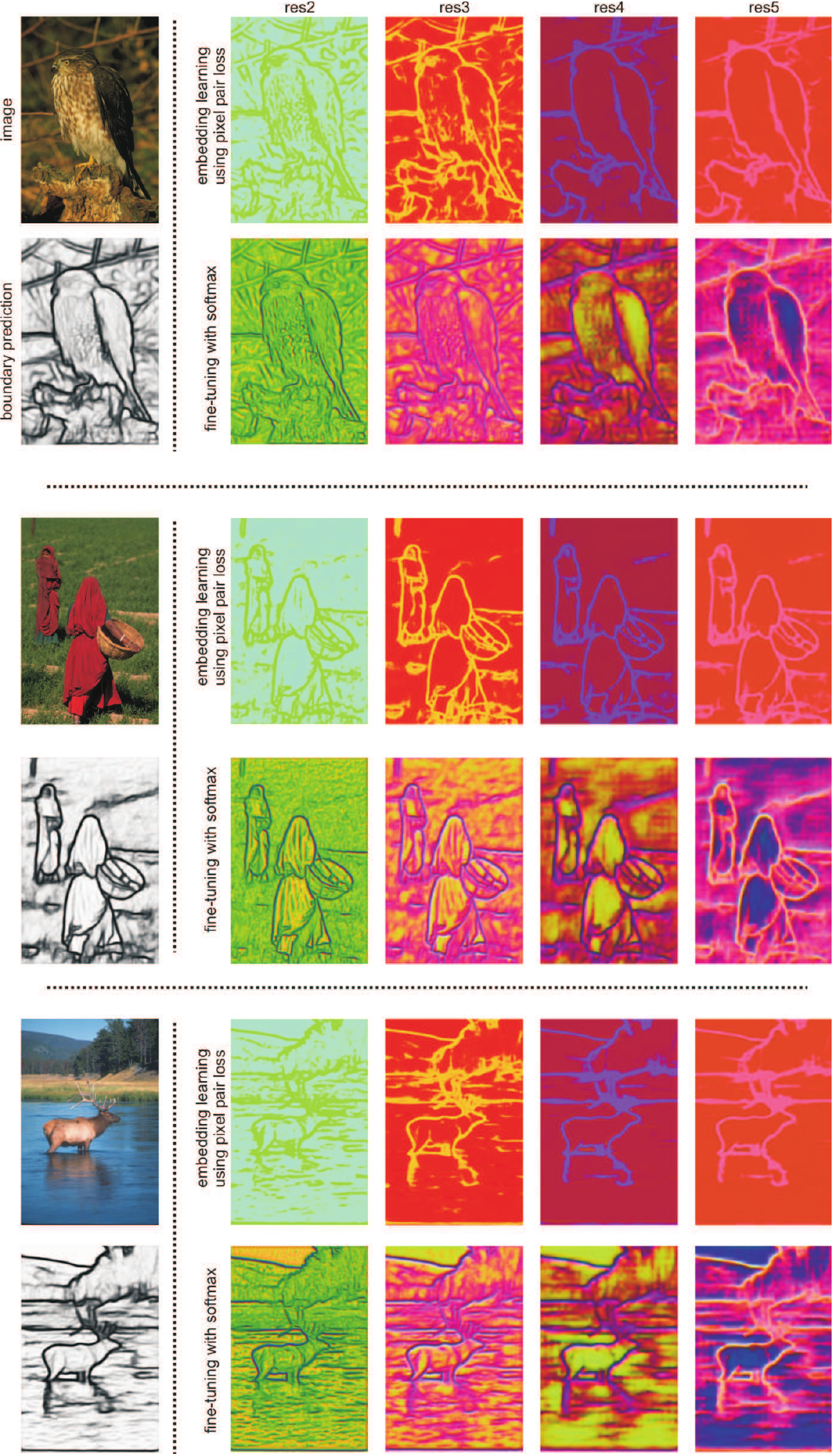}
   \vspace{-3mm}
   \caption{Visualization for boundary detection ($4/5$).
   Images are randomly selected from BSDS500 test set.  For each image, we show
   the embedding vectors  at different layers from the model before and after
   fine-tuning using logistic loss.  We can see that the boundary embedding
   vectors after fine-tuning not only highlights the boundary pixels, but also
   captures to some extent the edge orientation and distance from the colors
   conveyed.
   }
\label{fig:boundary_show_appendix_part4}
\end{figure*}

\begin{figure*}[t]
\centering
   \includegraphics[width=0.7\linewidth]{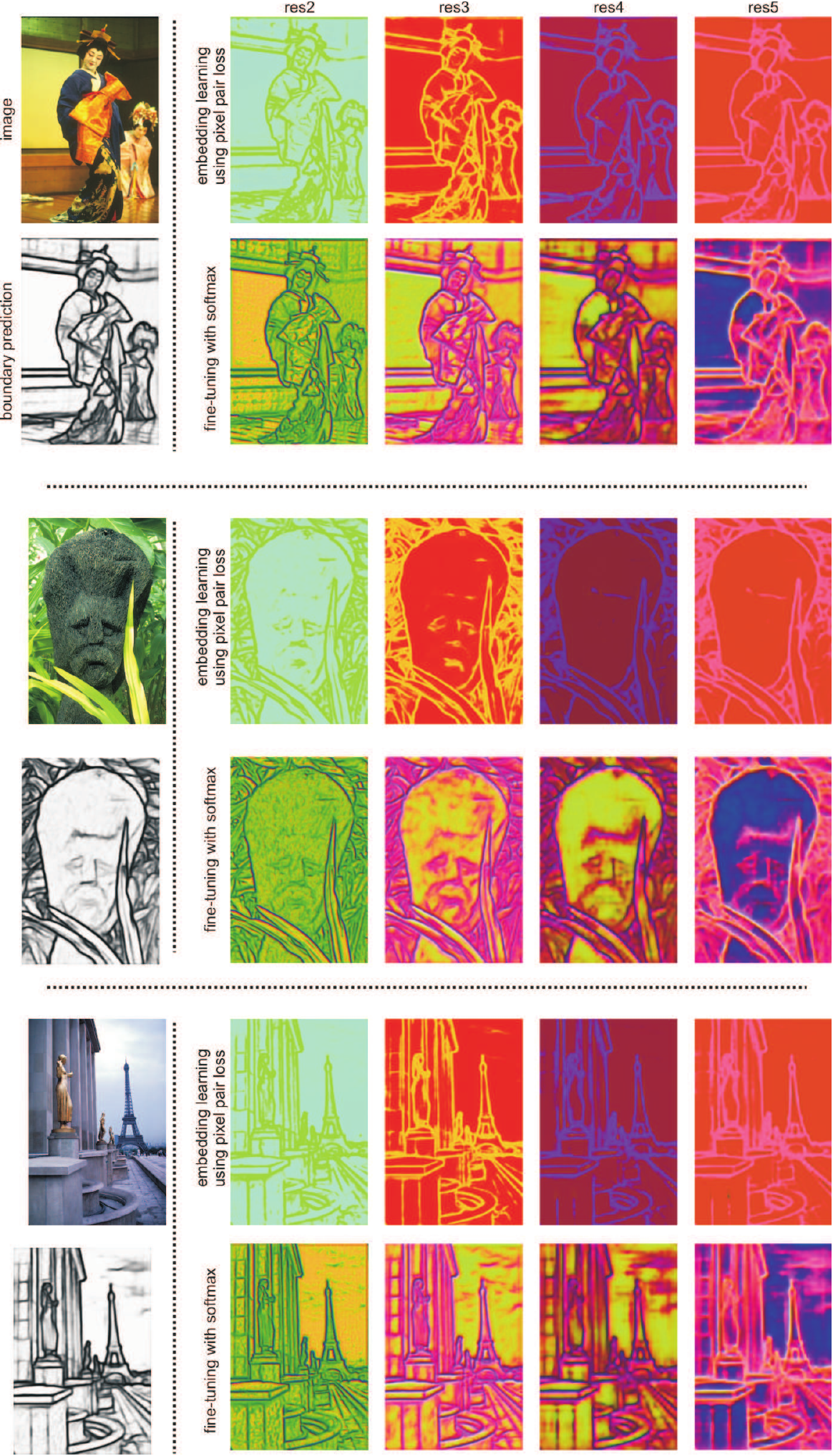}
   \vspace{-3mm}
   \caption{Visualization for boundary detection ($5/5$).
   Images are randomly selected from BSDS500 test set.  For each image, we show
   the embedding vectors  at different layers from the model before and after
   fine-tuning using logistic loss.  We can see that the boundary embedding
   vectors after fine-tuning not only highlights the boundary pixels, but also
   captures to some extent the edge orientation and distance from the colors
   conveyed.
   }
\label{fig:boundary_show_appendix_part5}
\end{figure*}

\section{Additional Results on Instance-Level Semantic Segmentation}

We show more instance-level semantic segmentation results on PASCAL VOC 2012
dataset~\cite{everingham2010pascal} based on our model in Figure
\ref{fig:instSeg_show_part1}, \ref{fig:instSeg_show_part2} and
\ref{fig:instSeg_show_part3}.  
As we learn 64-dimensional embedding
(hyper-sphere) space, to visualize the results, we randomly generate three
matrices to project the embeddings to 3-dimension vectors to be treated as RGB
images.  Besides showing the randomly projected embedding results, we also
visualize the semantic segmentation results used to product instance-level
segmentation.  From these figures, we observe the embedding for background
pixels are consistent, as the backgrounds have almost the same color.
Moreover, we can see the embeddings (e.g. in
Figure~\ref{fig:instSeg_show_part1}, the horses in row-7 and row-13, and the
motorbike in row-14) are able to connect the disconnected regions belonging to
the same instance.  Dealing with disconnected regions of one instance is an
unsolved problem for many methods, e.g. \cite{bai2016deep,
kirillov2016instancecut}, yet our approach has no problem with this situation.

\begin{figure*}[t]
\centering
   \includegraphics[width=0.95\linewidth]{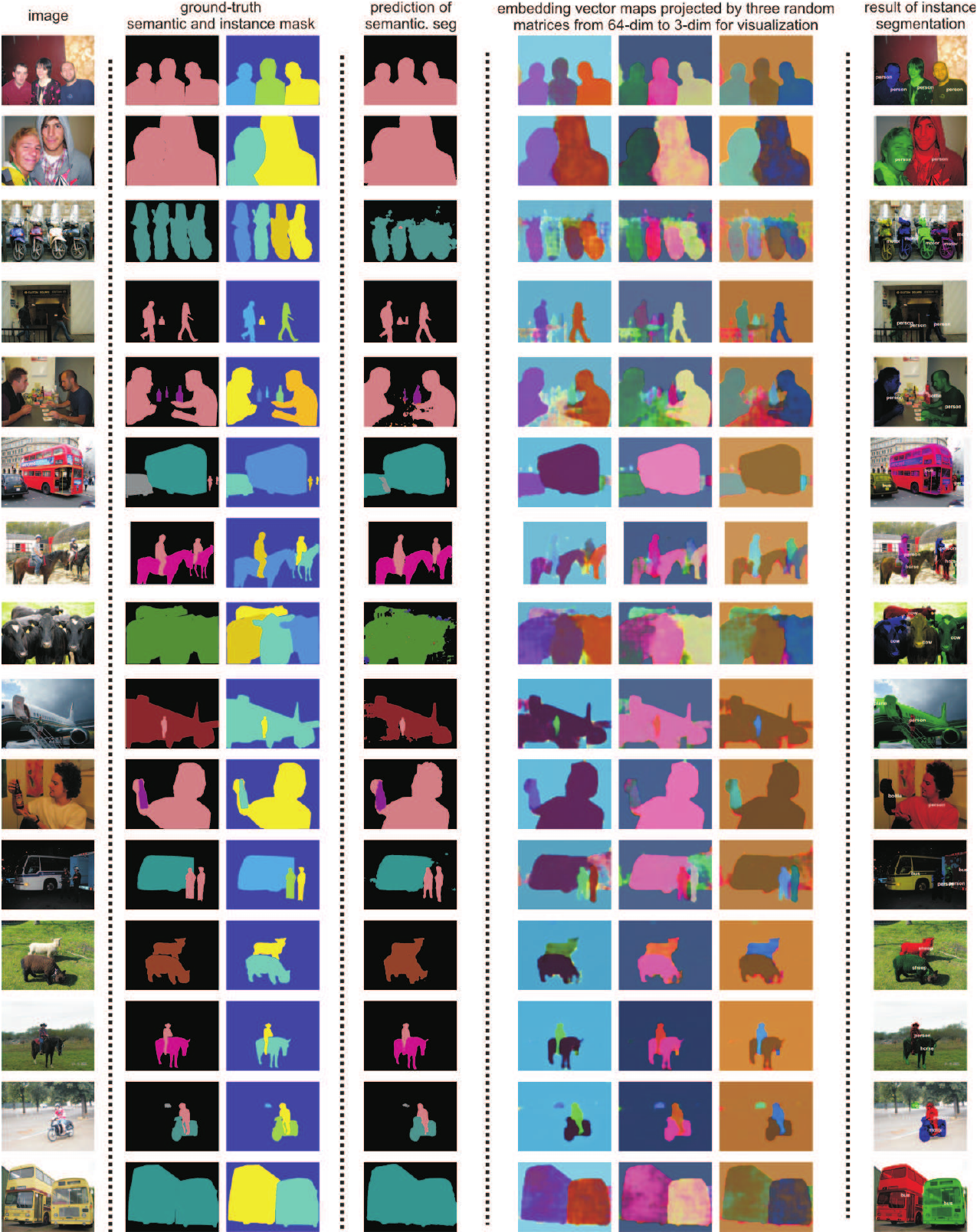}
   \caption{Visualization of generic and instance-level semantic segmentation
   with random projection of the embedding vectors (part-$1/3$).
   }
\label{fig:instSeg_show_part1}
\end{figure*}

\begin{figure*}[t]
\centering
   \includegraphics[width=0.95\linewidth]{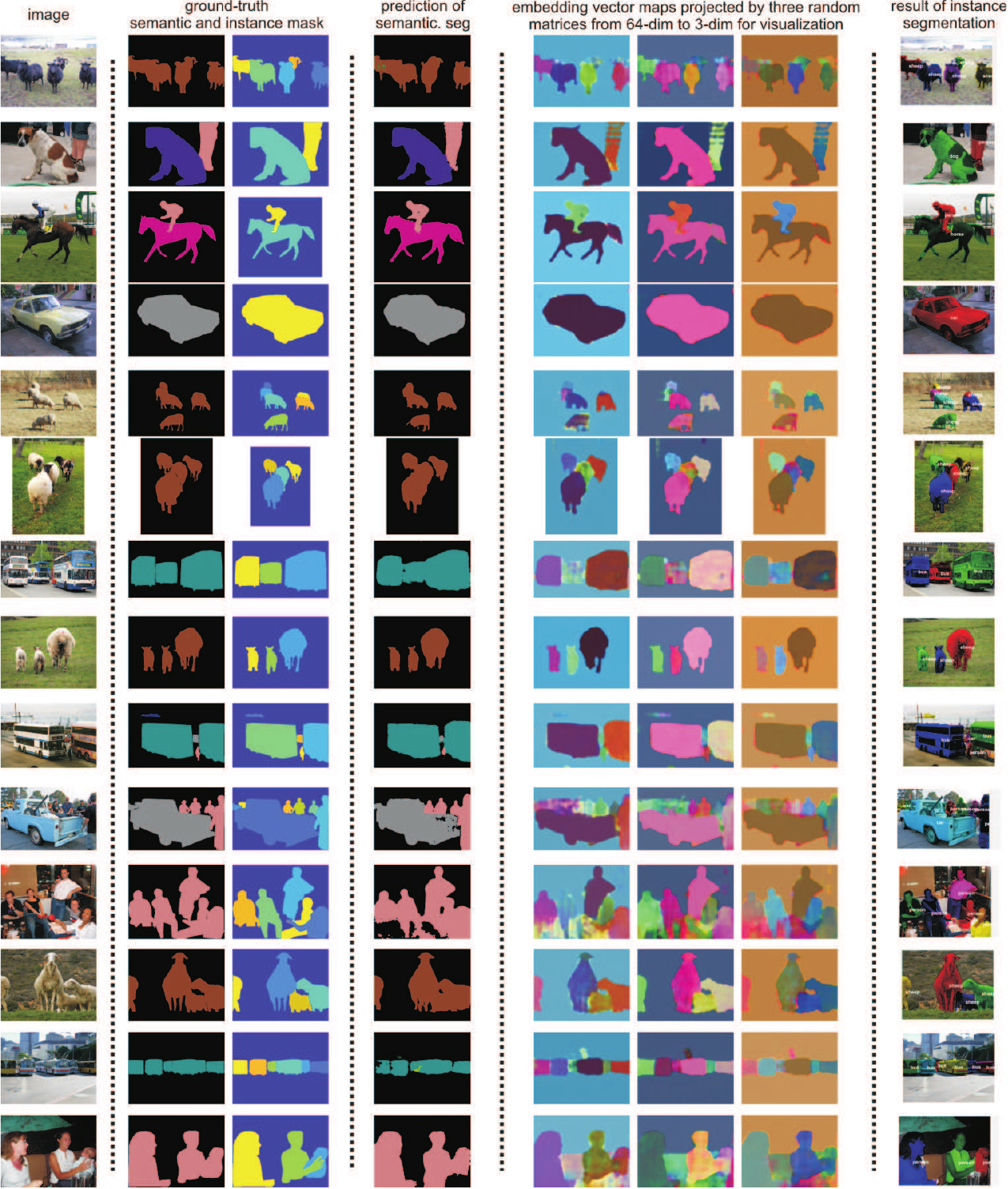}
   \caption{Visualization of generic and instance-level semantic segmentation
   with random projection of the embedding vectors (part-$2/3$).
   }
\label{fig:instSeg_show_part2}
\end{figure*}

\begin{figure*}[t]
\centering
   \includegraphics[width=0.95\linewidth]{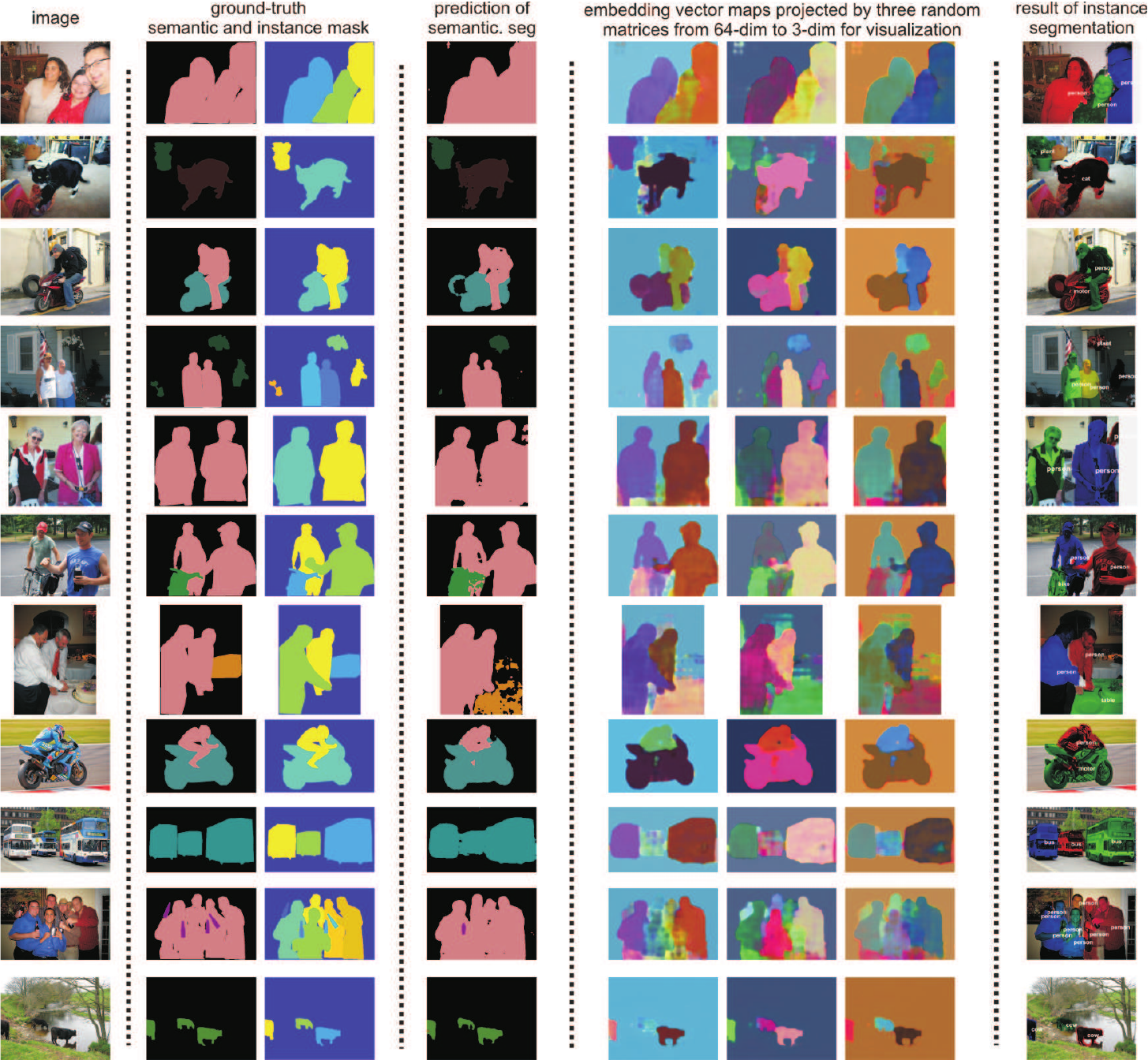}
   \caption{Visualization of generic and instance-level semantic segmentation
   with random projection of the embedding vectors (part-$3/3$).
   }
\label{fig:instSeg_show_part3}
\end{figure*}